\documentclass[11pt]{article}

\usepackage[preprint]{acl}

\usepackage{times}
\usepackage{latexsym}

\usepackage[T1]{fontenc}

\usepackage[utf8]{inputenc}

\usepackage{microtype}

\usepackage{inconsolata}

\usepackage{graphicx}

\usepackage{xurl}
\usepackage{amsthm}
\usepackage{amsmath}
\usepackage{amssymb}
\usepackage{booktabs}
\usepackage{multirow}
\usepackage{hyperref}
\theoremstyle{plain}
\usepackage[table]{xcolor}
\newtheorem{lemma}{Lemma}[section]

\title{ScaleSweep: Accurate NVFP4 Post-Training Quantization of LLMs via Block Scale Initialization}

\author{
 \textbf{Li Lin},
 \textbf{Xiaojun Wan},
\\
 Wangxuan Institute of Computer Technology, Peking University,
\\
 \small{
   efsotr\_l@stu.pku.edu.cn, wanxiaojun@pku.edu.cn
 }
}

\begin{document}
\maketitle
\begin{abstract}
NVFP4 is a recently introduced hardware-supported FP4 format that improves the fidelity of 4-bit quantization through fine-grained block scales. However, existing NVFP4 scale initialization methods still primarily rely on AbsMax initialization, which leaves a noticeable gap to the optimal solution.
To address this, we propose ScaleSweep, a simple and efficient scale optimization method that sweeps over feasible block scale candidates and selects the candidate that minimizes a target objective. We further provide a theoretical analysis of NVFP4 quantization and derive both lower and upper bounds for the required sweep range under mean square error (MSE) and weighted mean square error (WMSE) between the original tensor and the quantized reconstructed tensor. The proposed bounds substantially reduce the sweep space while preserving the optimal candidate, enabling negligible overhead compared with the baseline quantization operators.
Experiments on Llama and Qwen models demonstrate that ScaleSweep consistently improves quantization performance over existing initialization methods and further narrows the gap to full precision. In particular, under aggressive end-to-end quantization of weights, activations, KV cache, and query states, ScaleSweep preserves more than 93\% of the full-precision performance.
\end{abstract}

\section{Introduction}

Recent advances in large language models (LLMs) have substantially increased memory footprint, bandwidth demand, and computational cost during deployment. Post-training quantization (PTQ)~\citep{ptq} has therefore become a key approach for efficient inference, enabling model compression without retraining or full fine-tuning~\citep{gptq,smoothquant,quarot,spinquant,ostquant}. Among low-precision quantization schemes, NVFP4 is particularly notable for combining an FP4 E2M1 format with both FP8 micro-block scales and a tensor-level global scale, with native support on NVIDIA Blackwell GPUs~\citep{nvfp4_intro}. This combination reduces memory and bandwidth requirements while preserving greater numerical flexibility than integer-only formats~\citep{int_vs_fp,mr_gptq}. The FP8 micro-block scaling design of NVFP4 enables practical low-bit LLM inference under aggressive compression, making scale optimization increasingly critical in fine-grained low-precision quantization.

Despite the advantages of NVFP4, existing PTQ methods exhibit different behaviors under this format. Some methods, such as GPTQ~\citep{gptq} and SmoothQuant~\cite{smoothquant}, remain applicable to NVFP4, whereas rotation-based approaches~\citep{quarot,spinquant} may degrade performance~\citep{mr_gptq}. This difference stems from two key distinctions between NVFP4 and conventional INT4 quantization: the use of micro-block scaling and FP4 data types. Various scale initialization techniques have been proposed for INT quantization~\citep{leanquant,neuqi}, but they are not directly applicable to NVFP4 due to its two-level scaling structure. Existing NVFP4 initialization methods still primarily rely on AbsMax-based heuristics, including the 4/6 strategy~\citep{four_over_six}, where a noticeable gap to the optimal solution remains. These characteristics make scale optimization and error distribution in NVFP4 fundamentally different from those in both INT4 and single-level FP quantization, thereby motivating dedicated scale optimization methods for NVFP4.

\begin{figure*}[t]
  \centering
  \includegraphics[width=0.95\textwidth]{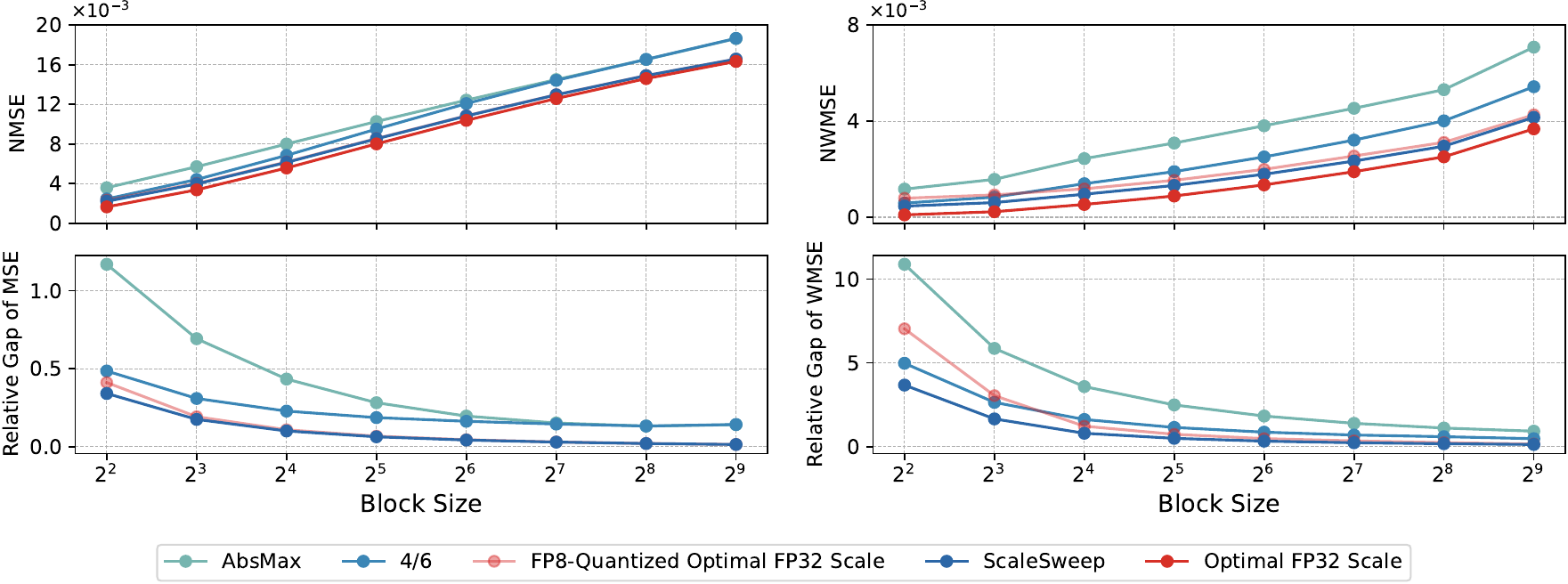}
  \caption{Normalized MSE and WMSE, together with their relative gaps to the optimal FP32 block scale, between the original tensor and the quantized reconstructed tensor under different NVFP4 block sizes using AbsMax, 4/6, ScaleSweep, and the FP8-quantized optimal FP32 scale. Definitions are provided in Section~\ref{sec:preliminary}.}
  \label{fig:vary_block_size}
\end{figure*}

Existing NVFP4 scale initialization strategies such as AbsMax and 4/6 rely on simple heuristics based on the maximum representable FP4 values. However, as shown in Figure~\ref{fig:vary_block_size}, they still exhibit a noticeable gap compared with the FP8-quantized optimal FP32 block scale\footnote{For both the MSE and WMSE objectives, the optimal FP32 block scale can be solved exactly with low computational complexity. Details are provided in Appendix~\ref{appendix:optimal-scale-for-fp4}.} across different block sizes. This observation suggests that there remains substantial room for improving FP8 block scale selection. Since the number of representable FP8 scales is very limited, exhaustive scale sweep becomes computationally feasible. To this end, we propose ScaleSweep, an NVFP4-specific scale sweep method designed for FP4 quantization with FP8 block scales. For FP4 quantization, we further provide a theorical analysis of block scale optimization under both MSE and WMSE objectives. In particular, through theorical analysis and computer-assisted analysis, we derive theoretically justified lower and upper bounds for the optimal FP8 block scale, thereby reducing the feasible sweep range to a compact local neighborhood in the FP8 bit-pattern space and enabling efficient scale sweep. We evaluate ScaleSweep under increasingly challenging quantization settings, including weight-activation quantization, weight-activation quantization with KV-cache quantization, and weight-activation quantization with both KV-cache and query state quantization. Across all settings, ScaleSweep generally achieves stronger recovery than existing initialization methods for NVFP4. 

Our main contributions are summarized below:
\begin{itemize}
\item We analyze FP4 quantization with FP8 block scales and derive lower bound and upper bound for the optimal block scale under MSE and WMSE objectives.

\item Based on the derived bounds, we propose ScaleSweep, an NVFP4-specific calibration method that restricts FP8 block scale optimization to a compact interval in the bit-pattern space, enabling efficient scale selection for RTN and GPTQ pipelines.

\item We validate the effectiveness of ScaleSweep on Llama and Qwen models across weight-activation, KV-cache, and query-state quantization settings, where ScaleSweep generally outperforms existing initialization methods and recovers up to $93$--$95\%$ of BF16 performance in the most aggressive setting, while introducing negligible operator overhead compared with the default NVFP4 quantization operators in vLLM.
\end{itemize}

\section{Related Work}

\paragraph{Integer quantization.}
Post-training quantization for integer quantization has been extensively studied for efficient large language model (LLM) inference. GPTQ~\cite{gptq} improves low-bit quantization through layer-wise reconstruction with approximate second-order information, while SmoothQuant~\citep{smoothquant} mitigates activation outliers through smoothing transformations between weights and activations. More recent work further improves low-bit quantization by reshaping tensor distributions before quantization. QuaRot~\citep{quarot} applies randomized Hadamard rotations to remove activation outliers and support 4-bit inference in rotated LLMs; SpinQuant~\citep{spinquant} learns rotation transformations to better align tensors with low-bit quantization grids; and OSTQuant~\citep{ostquant} combines orthogonal and scaling transformations to refine quantization through improved distribution fitting. Together, these methods show that reducing outliers and smoothing quantization-unfriendly distributions are central to accurate low-bit INT PTQ, where Hadamard, rotation, and orthogonal transformations have become increasingly important techniques.

\paragraph{FP4 quantization.}
FP4 quantization has recently emerged as an important direction for efficient low-precision LLM inference, particularly with the introduction of NVIDIA's NVFP4 format~\citep{nvfp4_intro}. Recent studies have begun to investigate FP4 quantization for both pretraining and post-training settings. NVIDIA's NVFP4 pretraining work demonstrates the feasibility of training large language models with NVFP4 precision~\citep{nvidia_pretrain_nvfp4}, while TetraJet-v2~\citep{tetrajetv2} further improves NVFP4 training accuracy by addressing weight oscillation and outlier issues during low-precision training. In terms of scale initialization, 4/6~\citep{four_over_six} extends AbsMax scaling by additionally evaluating a scale that maps the block maximum to 4 instead of 6 and selecting the lower-error quantization. MR-GPTQ~\citep{mr_gptq} shows that directly applying rotation transformations such as QuaRot and SpinQuant can degrade performance under NVFP4 quantization, and instead proposes micro-rotation on top of GPTQ for hardware-supported FP4 formats. These results suggest that FP4 quantization requires techniques tailored to floating-point codebooks and microscaling structures, particularly for scale selection and format-aware reconstruction, rather than a direct reuse of INT4 quantization methods.

\section{Preliminary} \label{sec:preliminary}

\subsection{Notation} \label{sec:notation}

Let $[N]$ denote the index set $\{0,1,2,\dots,N-1\}$. 
For an original tensor $x$ and its reconstructed tensor $\hat{x}$ of size $n$, the mean square error (MSE) and normalized mean square error (NMSE) are defined as $ \mathrm{MSE}(x,\hat{x})=\|x-\hat{x}\|_2^2/n$ and $ \mathrm{NMSE}(x,\hat{x})=\|x-\hat{x}\|_2^2/\|x\|_2^2 $, respectively. Given an element-wise weight tensor $w$, the weighted mean square error (WMSE) and normalized weighted mean square error (NWMSE) are defined as $ \mathrm{WMSE}(x,\hat{x},w)=(\sum_i w_i(x_i-\hat{x}_i)^2)/(\sum_i w_i)$ and $ \mathrm{NWMSE}(x,\hat{x},w)=(\sum_i w_i(x_i-\hat{x}_i)^2)/(\sum_i w_i x_i^2) $, respectively.

For a numeric format $\mathcal{F}$, let $\mathcal{G}_\mathcal{F}$ denote its representable value set. In particular, $\mathcal{G}_{\mathrm{FP4}}$, $\mathcal{G}_{\mathrm{FP8}}$, and $\mathcal{G}_{\mathrm{FP32}}$ denote the representable value set of FP4 E2M1, FP8 E4M3, and FP32, respectively, where $\mathcal{G}_{\mathrm{FP4}}$ and $\mathcal{G}_{\mathrm{FP8}}$ are also referred to as quantization grids.

For $x \in \mathbb{R}$, let $\left\lfloor x \right\rceil_{\mathcal{F}}$ denote rounding $x$ to the nearest value in $\mathcal{G}_{\mathcal{F}}$, let $\left\lfloor x \right\rfloor_{\mathcal{F}}$ denote rounding $x$ downward to the largest value in $\mathcal{G}_{\mathcal{F}}$ not exceeding $x$, and let $\left\lceil x \right\rceil_{\mathcal{F}}$ denote rounding $x$ upward to the smallest value in $\mathcal{G}_{\mathcal{F}}$ not smaller than $x$.

For FP4 quantization with scale $s$, define the scaled format $\mathrm{FP4}(s)$ as the standard FP4 format scaled by $s$. Its representable value set is
\begin{equation}
\mathcal{G}_{\mathrm{FP4}(s)} = \{ s \cdot v \mid v \in \mathcal{G}_{\mathrm{FP4}} \}.
\end{equation}

Let $\mathbf{w} \in \mathbb{R}_{\ge 0}^n$ denote non-negative weights. The weighted quantization loss under scale $s$ is
\begin{equation} \label{eq:weighted_quantization_loss}
\mathcal{L}(s; \mathbf{x}, \mathbf{w}) = \sum_{i=0}^{n-1} w_i \left(x_i - \left\lfloor x_i \right\rceil_{\mathrm{FP4}(s)}\right)^2.
\end{equation}
When $\mathbf{w}=\mathbf{1}$, the objective reduces to the unweighted loss
\begin{equation} \label{eq:mse_quantization_loss}
\mathcal{L}(s; \mathbf{x}) = \sum_{i=0}^{n-1} \left(x_i - \left\lfloor x_i \right\rceil_{\mathrm{FP4}(s)}\right)^2.
\end{equation}
For simplicity, both formulations are denoted by $\mathcal{L}$ when clear from context.

\subsection{NVFP4 Quantization}

NVFP4 combines FP4 values with two-level scaling to improve quantization fidelity under low bit-width constraints. The FP4 format used in NVFP4 follows the OCP Microscaling Formats Specification~\citep{ocp_mx} with representable values
\begin{equation}
\mathcal{G}_{\mathrm{FP4}} = \{0, \pm0.5, \pm1, \pm1.5, \pm2, \pm3, \pm4, \pm6\}.
\end{equation}

Given an input tensor $\mathbf{x} \in \mathbb{R}^N$, NVFP4 represents it as $(\mathbf{q}, \mathbf{s}, S)$, where $\mathbf{q} \in \mathcal{G}_{\mathrm{FP4}}^N$ denotes FP4 values, $\mathbf{s} \in \mathcal{G}_{\mathrm{FP8}}^{N/16}$ denotes FP8 E4M3 micro-block scales shared across every 16 elements, and $S \in \mathcal{G}_{\mathrm{FP32}}$ denotes a global FP32 scale. The reconstructed tensor $\hat{\mathbf{x}}$ is computed as
\begin{equation}
\hat{x}_i = q_i \cdot s_{\lfloor i/16 \rfloor} \cdot S.
\end{equation}

For NVFP4 quantization, a widely adopted initialization strategy is AbsMax~\citep{nvidia_pretrain_nvfp4}:
\begin{align}
S &= \frac{\max_i |x_i|}{448 \cdot 6}, \\
s_k &= \left\lfloor \frac{\max_{\lfloor i/16 \rfloor = k} |x_i|}{S \cdot 6} \right\rceil_{\mathrm{FP8}}, \  k \in [N/16], \\
q_i &= \left\lfloor \frac{x_i}{S \cdot s_{\lfloor i/16 \rfloor}} \right\rceil_{\mathrm{FP4}}, \  i \in [N],
\end{align}
where $448$ and $6$ are the maximum representable magnitudes of FP8 E4M3 and FP4 E2M1, respectively.

\section{Method}

\subsection{Optimization Objective}


\paragraph{MSE Objective}
For a weight tensor $W \in \mathbb{R}^{d_{\mathrm{in}} \times d_{\mathrm{out}}}$, reconstruction fidelity under NVFP4 quantization is naturally measured by the MSE objective. The resulting optimization objective is
\begin{align} \label{eq:nvfp4_MSE}
\mathcal{L}_{\mathrm{NVFP4}}^\mathrm{MSE}(\mathbf{s}, S; W)
&=
\sum_{i=0}^{d_{\mathrm{in}}/16-1}
\sum_{j=0}^{d_{\mathrm{out}}-1} \notag \\
&\quad
\mathcal{L}\!\Big(s_{i,j}\!\cdot\!S;\,
W_{16i:16(i+1),\,j}
\Big).
\end{align}

\paragraph{WMSE Objective}
For weight quantization, reconstruction error depends not only on the weight tensor itself, but also on the input activations propagated through the linear layer. Consider the local reconstruction error of the linear layer
\begin{align} \label{eq:linear_reconstruction_loss}
& \|XW - X\hat{W}\|_F^2 \notag \\
& = \mathrm{tr}\!\left((W-\hat{W})^T X^T X (W-\hat{W})\right),
\end{align}
where $X \in \mathbb{R}^{T \times d_{\mathrm{in}}}$ is the input activation matrix and $\hat{W} \in \mathbb{R}^{d_{\mathrm{in}} \times d_{\mathrm{out}}}$ is the quantized weight tensor.

Directly optimizing Eq.~\eqref{eq:linear_reconstruction_loss} is impractical for quantization parameter optimization because the full Gram matrix $X^T X$ couples quantization errors across different input channels. Therefore, we adopt the diagonal approximation commonly used in model compression and quantization~\citep{obd,squant,vptq}, retaining only the diagonal entries of $X^T X$. 
Under this approximation, the reconstruction objective becomes
\begin{equation}
\sum_{i=0}^{d_{\mathrm{in}}-1} \sum_{j=0}^{d_{\mathrm{out}}-1} \mathrm{Imp}_i \left(W_{i,j}-\hat{W}_{i,j}\right)^2,
\end{equation}
where the importance score of the $i$-th input channel is defined as
\begin{equation}
\mathrm{Imp}_i = (X_{:,i})^T X_{:,i} = \|X_{:,i}\|_2^2.
\end{equation}
The resulting formulation naturally induces the following WMSE objective for NVFP4 quantization:
\begin{align} \label{eq:nvfp4_WMSE}
\mathcal{L}_{\mathrm{NVFP4}}^{\mathrm{WMSE}}
&(\mathbf{s}, S; W, \mathrm{Imp}) \notag \\
&=
\sum_{i=0}^{d_{\mathrm{in}}/16-1}
\sum_{j=0}^{d_{\mathrm{out}}-1}
\mathcal{L}\!\Big(
s_{i,j}\!\cdot\!S; \notag \\
&\quad
W_{16i:16(i+1),\,j},
\mathrm{Imp}_{16i:16(i+1)}
\Big).
\end{align}



Similarly, for activation quantization in linear layers, the importance score of each input channel can be derived from the squared norm of the corresponding input channel in the associated weight matrix. For query state and KV cache quantization in attention operators, the attention output is computed as $\mathrm{softmax}(QK^T)VW^O$, where the inputs are the query states $Q$, key cache $K$, value cache $V$, and output projection matrix $W^O$. Accordingly, the importance scores for query states $Q$, key cache $K$, and value cache $V$ are derived from the input-channel squared norms of $K$, $Q$, and $W^O$, respectively.

\subsection{ScaleSweep}

\subsubsection{Global Scale Selection}

As shown in Eq.~\eqref{eq:nvfp4_MSE} and Eq.~\eqref{eq:nvfp4_WMSE}, once the global scale is fixed, the optimization of different micro-blocks becomes independent. The two-level scaling mechanism in NVFP4 can be interpreted as a form of double quantization, where block scales are quantized into FP8 values and the global scale acts as the corresponding FP8 quantization scale. Prior analyses of FP8 quantization~\citep{fp8_ptq,fp8_format_dl} show that the quantization error introduced by FP8 is relatively small. Therefore, it is sufficient to select a global scale that preserves adequate dynamic range for block scale optimization. Following~\citet{four_over_six}, we set
\begin{equation}
S = \frac{\max_i |x_i|}{256\cdot 6}.
\end{equation}
As illustrated in Figure~\ref{fig:vary_S}, different choices of global scale produce only marginal differences in normalized MSE for both AbsMax and ScaleSweep.

\begin{figure}[t]
  \includegraphics[width=0.95\columnwidth]{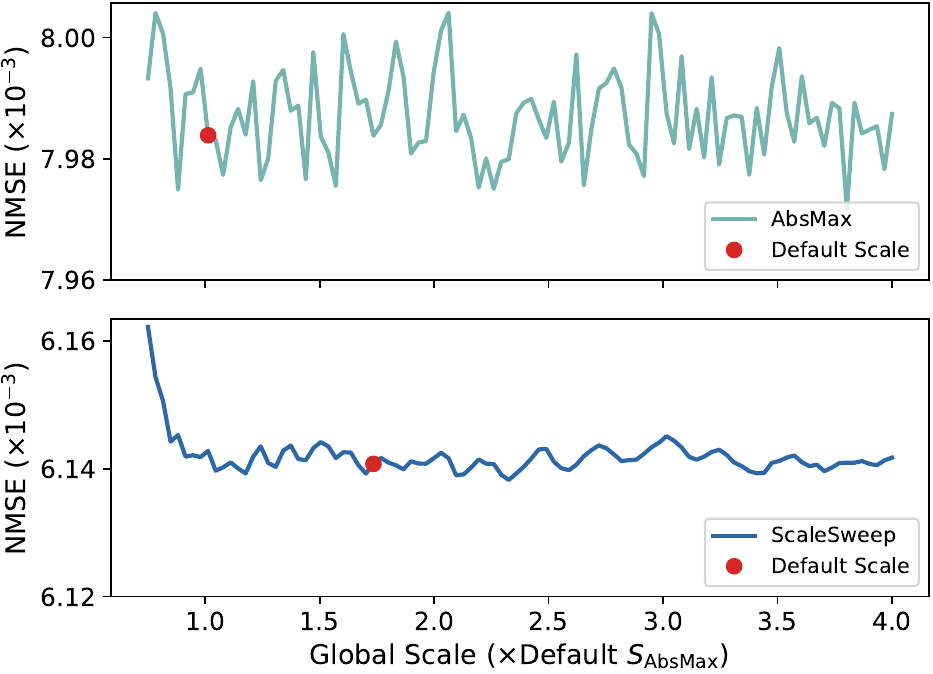}
  \caption{Normalized MSE between the original tensor and the quantized reconstructed tensor under different global scales for NVFP4, using AbsMax and ScaleSweep.}
  \label{fig:vary_S}
\end{figure}

\begin{table}[t]
\small
\centering
\begin{tabular}{ccc@{\qquad}ccc}
\toprule
$b$ & $\sup$ & $\Delta$ & $b$ & $\sup$ & $\Delta$ \\
\midrule
$000_2$ & $0\,111_2$ & $7$ & $100_2$ & $1\,011_2$ & $7$ \\
$001_2$ & $1\,000_2$ & $7$ & $101_2$ & $1\,100_2$ & $7$ \\
$010_2$ & $1\,001_2$ & $7$ & $110_2$ & $1\,100_2$ & $6$ \\
$011_2$ & $1\,010_2$ & $7$ & $111_2$ & $1\,101_2$ & $6$ \\
\bottomrule
\end{tabular}
\caption{
Bits-pattern upper bounds for $\alpha=\frac{12}{7}$ in E4M3. Here $b$ denotes the mantissa bits of $s_{\mathrm{base}}^{\mathrm{FP8}}$, $\mathrm{sup}$ denotes the carry bit concatenated with the mantissa bits of the largest possible $\left\lfloor \frac{12}{7}s_{\mathrm{base}} \right\rfloor_{\mathrm{E4M3}}$, and $\Delta=\mathrm{sup}-b$ denotes the upper bound on the bit difference relative to $b$.
}
\label{tab:every_mant_upper_bound}
\end{table}

\subsubsection{Block Scale Optimization}

For a single micro-block $\mathbf{x}=\{x_i\}_{i=0}^{n-1}$ with $n=16$, once the global scale $S$ is fixed, optimization can be equivalently performed on the normalized tensor $\mathbf{x}=\{x_i/S\}_{i=0}^{n-1}$, since the quantization loss differs only by a constant factor $S^2$. 

Since the block scale is represented in FP8 format, the number of positive finite representable values is limited to only $126$ candidates. Therefore, a natural strategy is therefore to sweep all candidate scales and select the one minimizing Eq.~\eqref{eq:weighted_quantization_loss}. However, sweeping over the entire FP8 scale space is inefficient. To improve optimization efficiency, it is desirable to derive tighter lower and upper bounds for the sweep range.
To this end, define the base scale as
\begin{equation}
s_{\mathrm{base}} = \frac{\max_i |x_i|}{6},
\quad
s_{\mathrm{base}}^{\mathrm{FP8}} = \left\lfloor s_{\mathrm{base}} \right\rfloor_{\mathrm{FP8}}.
\end{equation}

\paragraph{Block Scale Upper Bound}

\begin{lemma} \label{lem:upper_bound_of_scale}
If $s > \frac{\max_i |x_i|}{3.5} $, then
\begin{equation}
\mathcal{L}(s; \mathbf{x}, \mathbf{w}) \ge \mathcal{L}(s/2; \mathbf{x}, \mathbf{w}).
\end{equation}
\end{lemma}

Based on Lemma~\ref{lem:upper_bound_of_scale}, when $s > \frac{12}{7}s_{\mathrm{base}}$, the scale $s$ cannot achieve a lower quantization loss than $s/2$. Therefore, the upper bound of ScaleSweep is set to $\frac{12}{7}s_{\mathrm{base}}$. Using the same technique, the corresponding upper bound for non-FP4 formats can be derived as $2\times$ the baseline scale.

\begin{lemma} \label{lem:upper_bound_of_bits_diff}
Under the EeMm floating-point format introduced in Appendix~\ref{appendix:fp_format}, for any $1 \leq \alpha \leq 2$ and any positive $x$, the following inequality holds:
\begin{align}
& I_\mathrm{EeMm}\!\left(
\lfloor \alpha x \rfloor_\mathrm{EeMm}
\right) - I_\mathrm{EeMm}\!\left(\lfloor x \rfloor_\mathrm{EeMm}\right) \notag \\
& \leq \big\lceil (1-\frac{1}{\alpha}) 2^{m+1} \big\rceil,
\end{align}
where $I_\mathrm{EeMm}(\cdot)$ denotes the integer interpretation of the floating-point bit pattern of the EeMm value.
\end{lemma}

When the sweep upper bound is set to $\frac{12}{7}s_{\mathrm{base}}$, the corresponding upper bound in the bit-pattern space becomes $I_{E4M3}\!\left(s_{\mathrm{base}}^{\mathrm{FP8}}\right)+7$. This upper bound is obtained from Lemma~\ref{lem:upper_bound_of_bits_diff} by setting $\alpha=\frac{12}{7}$, and the resulting upper bound for each mantissa field is summarized in Table~\ref{tab:every_mant_upper_bound}.

\begin{table}[t]
\small
\centering
\begin{tabular}{ccc@{\qquad}ccc}
\toprule
$b$ & $\inf$ & $\Delta$ & $b$ & $\inf$ & $\Delta$ \\
\midrule
$000_2$ & $1\,101_2$ & $3$ & $100_2$ & $0\,010_2$ & $2$ \\
$001_2$ & $1\,111_2$ & $2$ & $101_2$ & $0\,011_2$ & $2$ \\
$010_2$ & $0\,000_2$ & $2$ & $110_2$ & $0\,100_2$ & $2$ \\
$011_2$ & $0\,001_2$ & $2$ & $111_2$ & $0\,100_2$ & $3$ \\
\bottomrule
\end{tabular}
\caption{
Bit-pattern lower bounds for FP8 E4M3 micro-block scales when $n=16$. Here $b$ denotes the mantissa bits of $s_{\mathrm{base}}^{\mathrm{FP8}}$, $\mathrm{inf}$ denotes the borrow bit concatenated with the mantissa bits of the smallest possible $\left\lceil \frac{4}{5}s_{\mathrm{base}}^\mathrm{FP8} \right\rceil_{\mathrm{E4M3}}$, and $\Delta$ denotes the upper bound on the bit difference relative to $b$.
}
\label{tab:every_mant_lower_bound}
\end{table}


\paragraph{Block Scale Lower Bound}

\begin{lemma} \label{lem:lower_bound_of_MSE}
For block size $16$, namely $n=16$, when $\frac{11}{7}s_{\mathrm{base}}^{\mathrm{FP8}} \le 448$, it follows that
\begin{equation}
\arg\min_{s\in \mathcal{G}_\mathrm{FP8}} \mathcal{L}(s;\mathbf{x})
\ge \frac{4}{5}s_{\mathrm{base}}^{\mathrm{FP8}}.
\end{equation}
\end{lemma}

For the MSE objective, based on Lemma~\ref{lem:lower_bound_of_MSE}, the lower bound of ScaleSweep can be set to $\frac{4}{5}s_{\mathrm{base}}^{\mathrm{FP8}}$. Accordingly, the corresponding lower bound in the bit-pattern space is  $I_{\mathrm{E4M3}}\!\left(s_{\mathrm{base}}^{\mathrm{FP8}}\right)-3$, and the resulting lower bound for each mantissa field is summarized in Table~\ref{tab:every_mant_lower_bound}.

For the WMSE objective, the optimal scale can become arbitrarily small if the weight associated with $\max_i |x_i|$ is sufficiently small. Therefore, ScaleSweep empirically sets the lower bound to $\frac{1}{2}s_{\mathrm{base}}$, and the corresponding lower bound in the bit-pattern space becomes $I_{\mathrm{E4M3}}\!\left(s_{\mathrm{base}}^{\mathrm{FP8}}\right)-8$.

In summary, for the MSE objective in Eq.~\eqref{eq:mse_quantization_loss}, the bit-pattern sweep space uses $I_{\mathrm{E4M3}}\!\left(s_{\mathrm{base}}^{\mathrm{FP8}}\right)$ as the zero point with offset range $[-3,7]$. For the WMSE objective in Eq.~\eqref{eq:weighted_quantization_loss}, the bit-pattern sweep space also uses $I_{\mathrm{E4M3}}\!\left(s_{\mathrm{base}}^{\mathrm{FP8}}\right)$ as the zero point, but with offset range $[-8,7]$.
The proofs of all lemmas are provided in Appendix~\ref{appendix:proofs}.

\section{Experiment}

\begin{table*}[th]
\small
\renewcommand{\arraystretch}{0.95}
\centering
\begin{tabular}{ccc@{\hskip 5pt}c@{\hskip 4pt}cccccc@{\hskip 3pt}c}
\toprule
\textbf{Setting} & \textbf{Method} & \textbf{Init Method} & \textbf{MMLU Pro} & \textbf{GSM8k} & \textbf{IFEval} & \textbf{HellaS} & \textbf{WinoG} & \textbf{Avg} & \textbf{Recovery (\%)} \\
\midrule
\multicolumn{10}{c}{\textbf{Llama-3.1-8B-Instruct}} \\
\midrule- & BF16 & - & 46.59 & 84.69 & 73.94 & 79.53 & 73.80 & 71.71 & 100 \\
\cmidrule(l){1-10}
\multirow[c]{8}{*}{WA} & \multirow[c]{4}{*}{RTN} & AbsMax & 41.48 & 78.32 & 69.69 & 77.91 & 72.53 & 67.99 & 94.81 \\
 &  & 4/6 & 42.29 & 80.59 & 67.28 & \textbf{78.52} & 73.01 & 68.34 & 95.30 \\
 &  & \cellcolor{blue!15}ScaleSweep$_\mathrm{MSE}$ & \cellcolor{blue!15}42.77 & \cellcolor{blue!15}\textbf{82.03} & \cellcolor{blue!15}\textbf{73.20} & \cellcolor{blue!15}78.17 & \cellcolor{blue!15}\textbf{73.09} & \cellcolor{blue!15}\textbf{69.85} & \cellcolor{blue!15}\textbf{97.41} \\
 &  & \cellcolor{blue!15}ScaleSweep & \cellcolor{blue!15}\textbf{43.09} & \cellcolor{blue!15}79.68 & \cellcolor{blue!15}70.79 & \cellcolor{blue!15}78.27 & \cellcolor{blue!15}72.45 & \cellcolor{blue!15}68.86 & \cellcolor{blue!15}96.03 \\
\cmidrule(l){2-10}
 & \multirow[c]{4}{*}{GPTQ} & AbsMax & 41.94 & \textbf{81.58} & 69.13 & 77.94 & 73.32 & 68.78 & 95.92 \\
 &  & 4/6 & 42.46 & 80.67 & 70.98 & 78.01 & 73.24 & 69.07 & 96.33 \\
 &  & \cellcolor{blue!15}ScaleSweep$_\mathrm{MSE}$ & \cellcolor{blue!15}\textbf{43.03} & \cellcolor{blue!15}81.43 & \cellcolor{blue!15}69.13 & \cellcolor{blue!15}78.19 & \cellcolor{blue!15}\textbf{74.11} & \cellcolor{blue!15}69.18 & \cellcolor{blue!15}96.47 \\
 &  & \cellcolor{blue!15}ScaleSweep & \cellcolor{blue!15}42.70 & \cellcolor{blue!15}81.50 & \cellcolor{blue!15}\textbf{74.49} & \cellcolor{blue!15}\textbf{78.35} & \cellcolor{blue!15}72.93 & \cellcolor{blue!15}\textbf{69.99} & \cellcolor{blue!15}\textbf{97.61} \\
\cmidrule(l){1-10}
\multirow[c]{8}{*}{WAKV} & \multirow[c]{4}{*}{RTN} & AbsMax & 38.21 & 76.95 & 68.95 & 77.31 & 71.35 & 66.55 & 92.81 \\
 &  & 4/6 & 39.29 & 78.70 & 66.54 & 77.58 & 72.14 & 66.85 & 93.23 \\
 &  & \cellcolor{blue!15}ScaleSweep$_\mathrm{MSE}$ & \cellcolor{blue!15}41.06 & \cellcolor{blue!15}78.85 & \cellcolor{blue!15}\textbf{70.24} & \cellcolor{blue!15}78.09 & \cellcolor{blue!15}\textbf{72.85} & \cellcolor{blue!15}68.22 & \cellcolor{blue!15}95.13 \\
 &  & \cellcolor{blue!15}ScaleSweep & \cellcolor{blue!15}\textbf{41.29} & \cellcolor{blue!15}\textbf{79.91} & \cellcolor{blue!15}69.69 & \cellcolor{blue!15}\textbf{78.12} & \cellcolor{blue!15}72.53 & \cellcolor{blue!15}\textbf{68.31} & \cellcolor{blue!15}\textbf{95.26} \\
\cmidrule(l){2-10}
 & \multirow[c]{4}{*}{GPTQ} & AbsMax & 38.40 & 77.94 & 70.61 & \textbf{77.95} & 71.35 & 67.25 & 93.78 \\
 &  & 4/6 & 40.26 & \textbf{80.06} & 69.13 & 77.78 & 72.22 & 67.89 & 94.68 \\
 &  & \cellcolor{blue!15}ScaleSweep$_\mathrm{MSE}$ & \cellcolor{blue!15}40.86 & \cellcolor{blue!15}79.45 & \cellcolor{blue!15}70.61 & \cellcolor{blue!15}77.71 & \cellcolor{blue!15}72.22 & \cellcolor{blue!15}68.17 & \cellcolor{blue!15}95.07 \\
 &  & \cellcolor{blue!15}ScaleSweep & \cellcolor{blue!15}\textbf{41.26} & \cellcolor{blue!15}79.98 & \cellcolor{blue!15}\textbf{71.90} & \cellcolor{blue!15}77.54 & \cellcolor{blue!15}\textbf{72.77} & \cellcolor{blue!15}\textbf{68.69} & \cellcolor{blue!15}\textbf{95.80} \\
\cmidrule(l){1-10}
\multirow[c]{8}{*}{WAKVQ} & \multirow[c]{4}{*}{RTN} & AbsMax & 34.04 & 69.90 & \textbf{70.61} & 76.99 & 70.88 & 64.48 & 89.93 \\
 &  & 4/6 & 36.69 & 74.22 & 68.02 & \textbf{77.43} & \textbf{72.14} & 65.70 & 91.62 \\
 &  & \cellcolor{blue!15}ScaleSweep$_\mathrm{MSE}$ & \cellcolor{blue!15}37.86 & \cellcolor{blue!15}74.98 & \cellcolor{blue!15}69.32 & \cellcolor{blue!15}77.38 & \cellcolor{blue!15}71.59 & \cellcolor{blue!15}66.23 & \cellcolor{blue!15}92.35 \\
 &  & \cellcolor{blue!15}ScaleSweep & \cellcolor{blue!15}\textbf{39.72} & \cellcolor{blue!15}\textbf{77.33} & \cellcolor{blue!15}69.13 & \cellcolor{blue!15}\textbf{77.43} & \cellcolor{blue!15}71.27 & \cellcolor{blue!15}\textbf{66.98} & \cellcolor{blue!15}\textbf{93.40} \\
\cmidrule(l){2-10}
 & \multirow[c]{4}{*}{GPTQ} & AbsMax & 34.53 & 73.16 & 68.21 & 76.92 & 71.51 & 64.87 & 90.46 \\
 &  & 4/6 & 36.45 & 76.57 & 67.65 & 77.62 & \textbf{72.14} & 66.09 & 92.16 \\
 &  & \cellcolor{blue!15}ScaleSweep$_\mathrm{MSE}$ & \cellcolor{blue!15}38.60 & \cellcolor{blue!15}\textbf{78.24} & \cellcolor{blue!15}68.58 & \cellcolor{blue!15}77.42 & \cellcolor{blue!15}72.06 & \cellcolor{blue!15}66.98 & \cellcolor{blue!15}93.41 \\
 &  & \cellcolor{blue!15}ScaleSweep & \cellcolor{blue!15}\textbf{39.67} & \cellcolor{blue!15}77.33 & \cellcolor{blue!15}\textbf{70.24} & \cellcolor{blue!15}\textbf{77.67} & \cellcolor{blue!15}71.82 & \cellcolor{blue!15}\textbf{67.35} & \cellcolor{blue!15}\textbf{93.92} \\
\bottomrule
\end{tabular}
\caption{Comparison results of ScaleSweep and baseline initialization methods under RTN and GPTQ across different settings on Llama-3.1-8B-Instruct.} \label{tab:main_result}
\end{table*}

\subsection{Experiment Setup}

\paragraph{Models and metrics.}
We evaluate four instruction-tuned large language models: Llama-3.1-8B-Instruct, Llama-3.2-3B-Instruct~\citep{llama31,llama32}, Qwen3-8B, and Qwen3-4B~\citep{qwen3}. All Qwen3 evaluations are conducted in non-thinking mode. We report results on five benchmarks to cover different aspects of model capability: MMLU Pro for world knowledge and reasoning with 5-shot prompting~\citep{mmlu_pro}, GSM8K for mathematical reasoning with 8-shot chain-of-thought prompting~\citep{gsm8k}, IFEval for instruction following in the 0-shot setting~\citep{ifeval}, and HellaSwag (HellaS)~\citep{hellaswag} together with WinoGrande (WinoG)~\citep{winogrande} for commonsense reasoning and language understanding in the 0-shot setting. \textbf{Avg} denotes the average score across the five benchmarks, and \textbf{Recovery (\%)} denotes the performance recovery relative to the BF16 baseline.

\paragraph{Quantization settings.}
We evaluate three quantization configurations. 
The first configuration, denoted as WA, applies NVFP4 quantization to both weights and activations. 
The second configuration, denoted as WAKV, further extends NVFP4 quantization to the KV cache. 
The third configuration, denoted as WAKVQ, additionally quantizes the query states using NVFP4. 
Together, these configurations enable a progressive evaluation of increasingly comprehensive NVFP4 quantization across model components.
All experiments are conducted with simulation quantization on NVIDIA L40 GPUs.

\paragraph{Baselines and initialization methods.}
We evaluate ScaleSweep under RTN and GPTQ~\citep{gptq} post-training quantization frameworks, where static activation reordering is applied for GPTQ. Calibration uses 128 sequences of length 2048 sampled from RedPajama~\citep{redpajama}. For scale initialization and optimization baselines, we compare against AbsMax and 4/6~\citep{four_over_six}. Additional implementation details are provided in Appendix~\ref{appendix:implementation_details}.

\begin{figure}[t]
  \includegraphics[width=0.95\columnwidth]{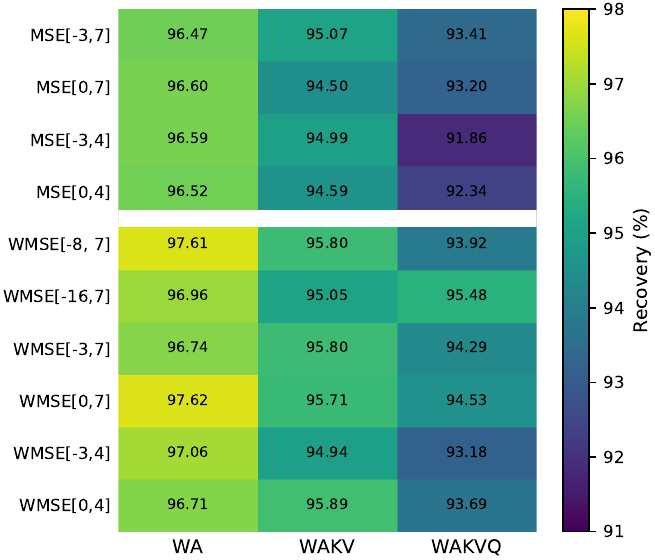}
  \caption{Heatmaps of ScaleSweep under different lower and upper bounds in the bit-pattern space for the MSE and WMSE objectives. $\mathrm{MSE}[l,r]$ and $\mathrm{WMSE}[l,r]$ denote ScaleSweep on Llama-3.1-8B-Instruct using the MSE and WMSE objectives, respectively, with bit-pattern sweep bounds set to $[l,r]$.}
  \label{fig:scale_sweep_range}
\end{figure}

\subsection{Results}

As shown in Table~\ref{tab:main_result} and Table~\ref{tab:main_result_llama_3_2_3B},\ref{tab:main_result_qwen_3_4B_8B} in Appendix~\ref{appendix:full_result}, ScaleSweep and ScaleSweep$_\mathrm{MSE}$ achieve competitive performance across different NVFP4 quantization settings and, in most cases, further narrow the gap to the BF16 baseline compared with AbsMax and 4/6 initialization. Moreover, ScaleSweep consistently achieves further improvements over ScaleSweep$_\mathrm{MSE}$ in most settings.
Under the WA setting, the baseline initialization methods already achieve strong performance, while ScaleSweep still provides additional improvements for both RTN and GPTQ quantization. In particular, on Qwen3-8B, ScaleSweep further improves the recovery rate to $99.50\%$.
Under the WAKV setting, ScaleSweep continues to provide incremental gains on top of already competitive baselines, where the recovery rate is improved by around $1\%$ to $2\%$ across different models and quantization backends in most cases.
Under the more challenging WAKVQ setting, ScaleSweep demonstrates stronger robustness as the quantization difficulty increases, while maintaining high recovery rates across different models. In particular, Llama-3.2-3B-Instruct maintains at least $93\%$ recovery rate under both RTN and GPTQ quantization, while Qwen3-4B consistently achieves at least $94.5\%$. These results demonstrate that ScaleSweep remains effective under increasingly aggressive quantization settings.

\paragraph{Scale Sweep Range}
We further evaluate the impact of different lower and upper bounds in ScaleSweep under both MSE and WMSE objectives. As shown in Figure~\ref{fig:scale_sweep_range}, a narrower sweep range does not necessarily lead to worse performance and can occasionally achieve slightly better results. Noticeable degradation appears only under more aggressive quantization settings. This phenomenon suggests that the local surrogate metric used in ScaleSweep cannot perfectly characterize the final downstream performance, which remains an important direction for future research.

\begin{table}[ht]
\renewcommand{\arraystretch}{0.95}
\small
\centering
\begin{tabular}{ccccc}
\toprule
\textbf{Setting} & \textbf{Method} & \textbf{Dataset} & \textbf{GSM8k} & \textbf{Avg}  \\
\midrule
\multirow[c]{2}{*}{WA} & \multirow[c]{2}{*}{GPTQ} & RedPajama & 81.50 & 69.99 \\
 &  & GSM8k & \textbf{83.47} & \textbf{70.78} \\
\midrule
\multirow[c]{2}{*}{WAKV} & \multirow[c]{2}{*}{GPTQ} & RedPajama & 79.98 & 68.69 \\
 &  & GSM8k & \textbf{81.88} & \textbf{68.84} \\
\midrule
\multirow[c]{2}{*}{WAKVQ} & \multirow[c]{2}{*}{GPTQ} & RedPajama & 77.33 & 67.35 \\
 &  & GSM8k & \textbf{80.21} & \textbf{67.79} \\
\bottomrule
\end{tabular}
\caption{Comparison results of ScaleSweep under different settings on Llama-3.1-8B-Instruct using RedPajama and the GSM8K training split as calibration datasets.} \label{tab:calib_GSM8k}
\end{table}

\paragraph{Calibration Data}
We further assess the impact of using the GSM8K training split rather than RedPajama as the calibration dataset for GPTQ with ScaleSweep. As reported in Table~\ref{tab:calib_GSM8k}, the performance advantage on GSM8K becomes more pronounced under increasingly aggressive quantization, yielding an improvement of approximately $3$ points in the WAKVQ setting. Across the five evaluated benchmarks, the average performance also exhibits a modest improvement, primarily attributable to gains on GSM8K, while the results on the remaining four benchmarks remain largely consistent.

\begin{figure*}[ht]
  \centering
  \includegraphics[width=0.9\textwidth]{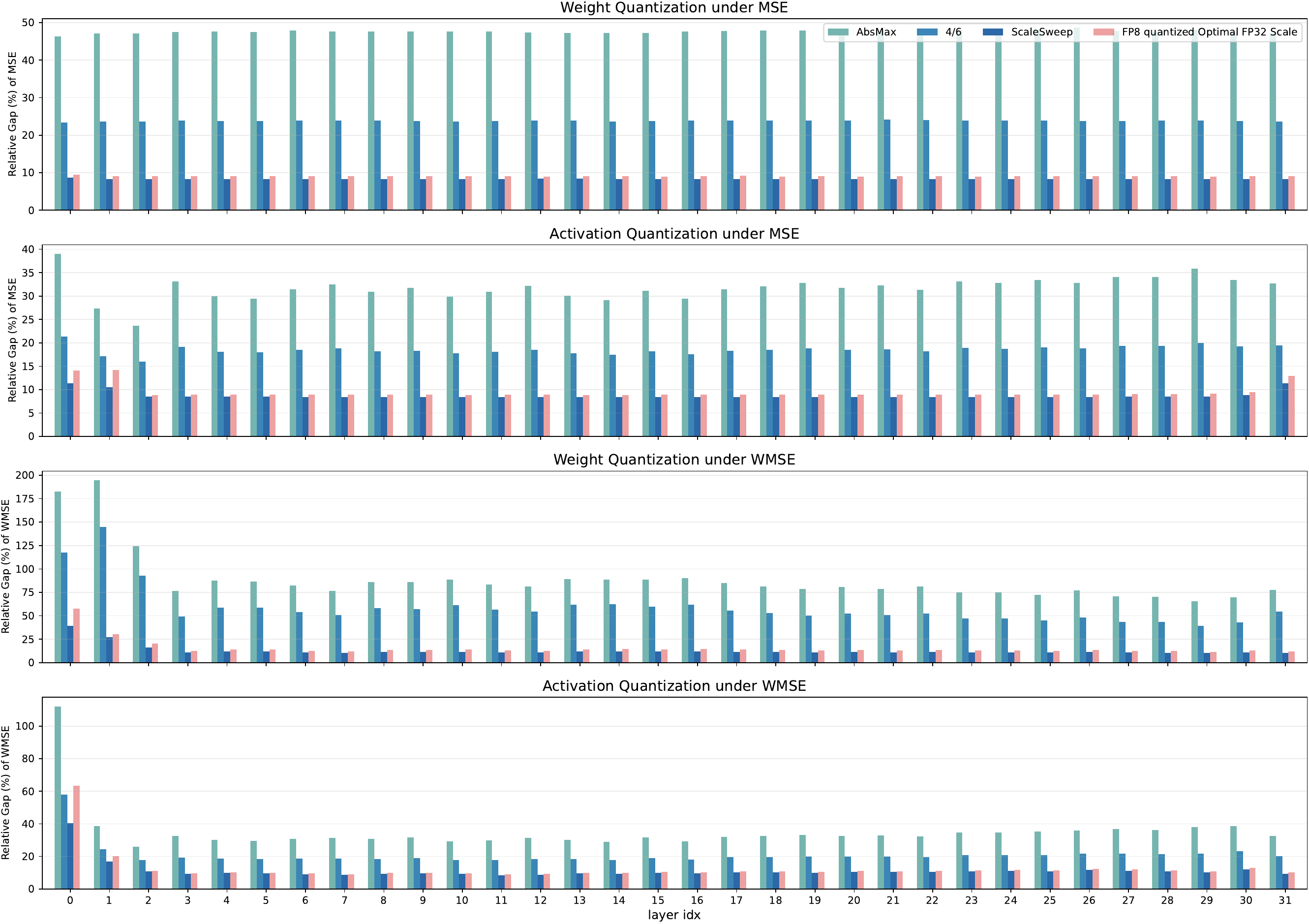}
  \caption{Performance gap (\%) of ScaleSweep, initialization baselines, and FP8-quantized optimal FP32 scales relative to the optimal FP32 block scale in terms of MSE and WMSE for weight quantization and activation quantization on Llama-3.1-8B-Instruct. For both weight and activation quantization, results are averaged over all quantized tensors within each Transformer block.}
  \label{fig:mse_and_wmse_on_each_layer_WA}
\end{figure*}

\paragraph{Quantization Error Analysis} 

We evaluate the quantization error of different initialization methods against the optimal FP32 block scales across layers. As shown in Figure~\ref{fig:mse_and_wmse_on_each_layer_WA} and Figure~\ref{fig:mse_and_wmse_on_each_layer_QKV} in Appendix~\ref{appendix:full_result}, ScaleSweep achieves substantial improvements over AbsMax and 4/6, and in most cases even outperforms the FP8-quantized optimal FP32 scale. For the MSE objective, ScaleSweep achieves a relative gap below $10\%$ in nearly all cases. For the WMSE objective, the relative gap also remains below $10\%$ in nearly all cases for weight quantization, activation quantization, and value cache quantization.

\begin{figure}[t]
  \includegraphics[width=0.95\columnwidth]{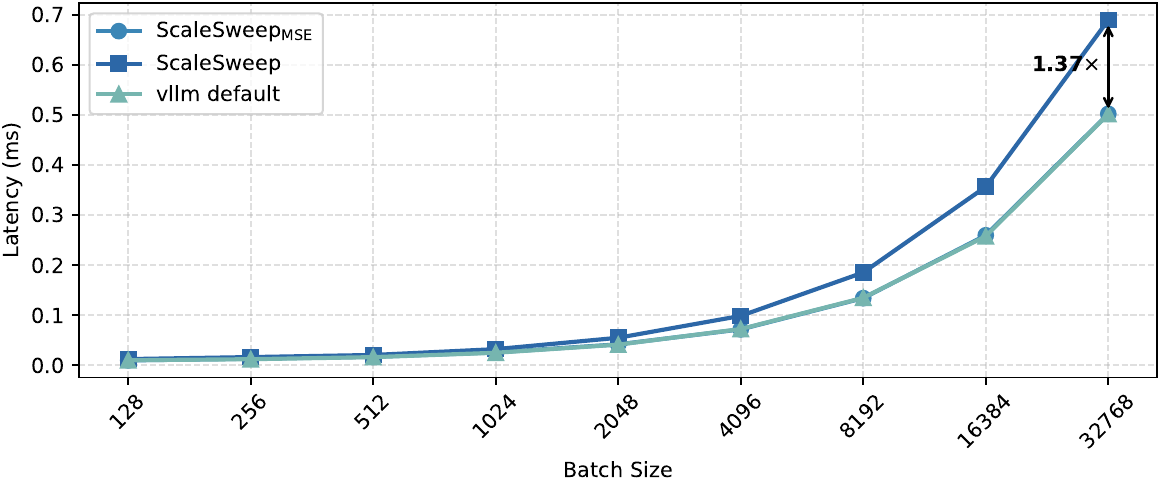}
  \caption{NVFP4 quantization operator latency comparison across different batch sizes with hidden dimension 8192 on NVIDIA RTX PRO 6000 Blackwell GPUs.}
  \label{fig:kernel_latency}
\end{figure}

\paragraph{Overhead Analysis} 
We further implemented preliminary Triton kernels for ScaleSweep and ScaleSweep$_\mathrm{MSE}$, and evaluated them under different batch sizes with hidden dimension 8192 against the default NVFP4 quantization operator in vLLM on NVIDIA RTX PRO 6000 Blackwell GPUs. The comparison results are presented in Figure~\ref{fig:kernel_latency} and Table~\ref{tab:kernel_latency} in Appendix~\ref{appendix:full_result}. ScaleSweep$_\mathrm{MSE}$ achieves nearly identical latency to the default vLLM operator, while ScaleSweep incurs only approximately $1.37\times$ higher latency, indicating negligible overhead during practical inference. Since the operators are memory-bound, further optimization through kernel fusion remains feasible.

\section{Conclusion}

In this paper, we propose ScaleSweep, a scale sweep method for FP8 block scales in NVFP4 quantization that can be seamlessly integrated with existing PTQ methods such as RTN and GPTQ. To the best of our knowledge, our study is the first to derive lower and upper bounds for block scale optimization under both the MSE and WMSE objectives in NVFP4 quantization. Based on these bounds, ScaleSweep restricts the sweep range to a theoretically justified local neighborhood in the FP8 bit-pattern space. Experimental results demonstrate that ScaleSweep consistently outperforms existing initialization baselines in weight-activation quantization, and further improves performance under more aggressive settings involving KV cache quantization and query state quantization. Further experiments demonstrate that ScaleSweep introduces negligible overhead during inference. 


\section*{Limitations}

In this work, we focus on scale optimization for NVFP4 quantization. Accordingly, the evaluation is restricted to NVFP4 settings and may not directly generalize to other low-precision formats or quantization schemes with different scaling structures. In addition, the effectiveness of quantization can vary depending on model characteristics and calibration data distributions, and the behavior of ScaleSweep in broader application scenarios remains to be further investigated.

\section*{Ethical considerations}

This work does not present ethical concerns. All large language models and datasets used in this study are publicly available and utilized strictly for research purposes in accordance with their licenses. The datasets are carefully examined to avoid personally identifiable information and offensive content. AI assistants are used only for polishing descriptive text and generating plotting code.


\bibliography{custom}

\appendix

\section{Floating-Point Representation} \label{appendix:fp_format}

EeMm denotes a floating-point format consisting of one sign bit, $e$ exponent bits, and $m$ mantissa bits. Unless otherwise specified, the exponent bias is defined as
\begin{equation}
B = 2^{e-1} - 1 .
\end{equation}

Only finite floating-point numbers are considered throughout this work. Therefore, the special encodings corresponding to infinity and NaN are excluded.

For a floating-point number with sign bit $c$, exponent field $a$, and mantissa field $b$, the represented finite value is denoted by $v(c,a,b)$ and defined as
\begin{equation}
v(c,a,b)=
\begin{cases}
(-1)^c 2^{a-B} \left(1+\dfrac{b}{2^m}\right), & a \neq 0, \\\\
(-1)^c 2^{1-B}\dfrac{b}{2^m}, & a = 0,
\end{cases}
\end{equation}
where $c \in \{0, 1\}$, $a \in [2^e]$, and $b \in [2^m]$.

The set of all representable finite values in the EeMm format is denoted by
\begin{align}
 \mathcal{G}_\mathrm{EeMm} = \{ v(c,a,b) \mid & \ c \in \{0,1\},\ a \in [2^e], \notag \\ 
&\ b \in [2^m] \}.
\end{align}

The corresponding integer interpretation of the bit pattern of the value is defined as
\begin{equation}
I_\mathrm{EeMm}(v(c,a,b)) = c \cdot 2^{e+m} + a \cdot 2^m + b .
\end{equation}

\section{Proofs} \label{appendix:proofs}

\subsection{Proof of Lemma~\ref{lem:upper_bound_of_scale}}

\begin{proof}
Since $s>\max_i |x_i|/3.5$, every element satisfies $|x_i|<3.5s$. Under nearest-point FP4 quantization with scale $s$, no element can therefore be rounded to $\pm4s$, whose nearest decision boundary begins at $3.5s$. Consequently, every activated quantized value under $\mathrm{FP4}(s)$ belongs to
\begin{equation}
\{0,\pm0.5s,\pm s,\pm1.5s,\pm2s,\pm3s\}.
\end{equation}
This set is entirely contained in $\mathcal{G}_{\mathrm{FP4}(s/2)}$. Hence, for each $x_i$, the quantized value $\left\lfloor x_i \right\rceil_{\mathrm{FP4}(s)}$ is also feasible under $\mathrm{FP4}(s/2)$. By the optimality of nearest-point quantization,
\begin{equation}
\left|x_i-\left\lfloor x_i \right\rceil_{\mathrm{FP4}(s/2)}\right|^2 \le \left|x_i-\left\lfloor x_i \right\rceil_{\mathrm{FP4}(s)}\right|^2.
\end{equation}
Multiplying both sides by $w_i\ge0$ and summing over all elements yields
\begin{equation}
\mathcal{L}(s/2;\mathbf{x},\mathbf{w}) \le \mathcal{L}(s;\mathbf{x},\mathbf{w}).
\end{equation}
\end{proof}

\subsection{Proof of Lemma~\ref{lem:upper_bound_of_bits_diff}}

\begin{proof}
Let the exponent field have \(e\) bits and the mantissa field have \(m\) bits. Define $M=2^m$.

For any positive \(EeMm\) number, let \(a\) denote the exponent field and let \(b\) denote the mantissa field, where $0\leq b<M$. If the number is exactly representable in \(EeMm\), then \(b\) is an integer. Since the proof also considers arbitrary positive real numbers \(x\), this notation is extended to non-representable values by allowing \(b\) to be real.

\paragraph{Subnormal case.}
Assume that \(x\) lies in the subnormal range. Then
\begin{equation}
I_\mathrm{EeMm}\!\left(\lfloor x \rfloor_\mathrm{EeMm}\right)=\lfloor b\rfloor .
\end{equation}

Since \(\alpha\leq 2\), we have \(\alpha b<2M\), and therefore
\begin{equation}
I_\mathrm{EeMm}\!\left(\lfloor \alpha x \rfloor_\mathrm{EeMm}\right)\leq \lfloor \alpha b\rfloor .
\end{equation}
Hence,
\begin{align}
& I_\mathrm{EeMm}\!\left(\lfloor \alpha x \rfloor_\mathrm{EeMm}\right)-I_\mathrm{EeMm}\!\left(\lfloor x \rfloor_\mathrm{EeMm}\right) \notag \\
& \leq \lfloor \alpha b\rfloor-\lfloor b\rfloor .
\end{align}

Using
\begin{equation}
\lfloor \alpha b\rfloor-\lfloor b\rfloor\leq \lceil (\alpha-1)b\rceil ,
\end{equation}
together with \(b<M\), we obtain
\begin{equation}
\lfloor \alpha b\rfloor-\lfloor b\rfloor\leq \left\lceil (\alpha-1)M \right\rceil .
\end{equation}

Since \(1\leq \alpha\leq 2\),
\begin{equation}
(\alpha-1)M\leq \left(1-\frac{1}{\alpha}\right)2M .
\end{equation}
Therefore,
\begin{align}
& I_\mathrm{EeMm}\!\left(\lfloor \alpha x \rfloor_\mathrm{EeMm}\right)
- I_\mathrm{EeMm}\!\left(\lfloor x \rfloor_\mathrm{EeMm}\right) \notag \\
& \leq \left\lceil \left(1-\frac{1}{\alpha}\right)2M \right\rceil .
\end{align}

\paragraph{Normal case below the binade boundary.}
Assume that \(x\) lies in the normal range. Then
\begin{equation}
I_\mathrm{EeMm}\!\left(\lfloor x \rfloor_\mathrm{EeMm}\right)=aM+\lfloor b\rfloor .
\end{equation}

Suppose that multiplying by \(\alpha\) does not cross into the next exponent interval:
\begin{equation}
\alpha(M+b)\leq 2M .
\end{equation}
Then
\begin{equation}
I_\mathrm{EeMm}\!\left(\lfloor \alpha x \rfloor_\mathrm{EeMm}\right)
\leq
aM+\left\lfloor \alpha(M+b)-M \right\rfloor .
\end{equation}
Therefore,
\begin{align}
& I_\mathrm{EeMm}\!\left(\lfloor \alpha x \rfloor_\mathrm{EeMm}\right)
- I_\mathrm{EeMm}\!\left(\lfloor x \rfloor_\mathrm{EeMm}\right) \notag \\
& \leq \left\lfloor b+(\alpha-1)(M+b) \right\rfloor-\lfloor b\rfloor .
\end{align}

Using
\begin{align}
& \left\lfloor b+(\alpha-1)(M+b) \right\rfloor-\lfloor b\rfloor \notag \\
& \leq \left\lceil (\alpha-1)(M+b) \right\rceil,
\end{align}
together with \(\alpha(M+b)\leq 2M\), we obtain
\begin{equation}
\left\lfloor b+(\alpha-1)(M+b) \right\rfloor-\lfloor b\rfloor
\leq
\left(1-\frac{1}{\alpha}\right)2M .
\end{equation}
Hence,
\begin{align}
& I_\mathrm{EeMm}\!\left(\lfloor \alpha x \rfloor_\mathrm{EeMm}\right)
- I_\mathrm{EeMm}\!\left(\lfloor x \rfloor_\mathrm{EeMm}\right) \notag \\
& \leq \left\lceil \left(1-\frac{1}{\alpha}\right)2M \right\rceil .
\end{align}

\paragraph{Normal case beyond the binade boundary.}

Now assume $\alpha (M+b)>2M$. At the binade boundary \(\alpha(M+b)=2M\), the desired upper bound has already been obtained in the previous case. Moreover, \(\lfloor \alpha x \rfloor_\mathrm{EeMm}\) is exactly the first representable value on this boundary. After this point, increasing \(I_\mathrm{EeMm}\!\left(\lfloor \alpha x \rfloor_\mathrm{EeMm}\right)\) by one requires an increase of \(2/\alpha\) in the mantissa field of \(x\), while increasing \(I_\mathrm{EeMm}\!\left(\lfloor x \rfloor_\mathrm{EeMm}\right)\) by one requires an increase of \(1\). Since \(2/\alpha\geq 1\), the difference cannot exceed its value at the boundary. Hence,
\begin{align}
& I_\mathrm{EeMm}\!\left(\lfloor \alpha x \rfloor_\mathrm{EeMm}\right)
- I_\mathrm{EeMm}\!\left(\lfloor x \rfloor_\mathrm{EeMm}\right) \notag \\
& \leq \left\lceil \left(1-\frac{1}{\alpha}\right)2M \right\rceil .
\end{align}

\end{proof}

\subsection{Proof of Lemma~\ref{lem:lower_bound_of_MSE}}

\begin{proof}

The proof consists of first excluding all ratios satisfying $\le \tfrac34$, and then eliminating the remaining ratios in the interval $(\tfrac34,\tfrac45)$ one by one.

\paragraph{Notation}
Let
\begin{equation}
s_0 = s_{\mathrm{base}}^{\mathrm{FP8}} 
= \left\lfloor \frac{\max_i |x_i|}{6} \right\rfloor_{\mathrm{FP8}}.
\end{equation}
By symmetry of the FP4 value set, it suffices to consider only nonnegative values. Define the nonnegative FP4 E2M1 grids as
\begin{equation}
\mathcal{G}=\left\{0,\tfrac12,1,\tfrac32,2,3,4,6\right\}.
\end{equation}
The normalized single-coordinate quantization loss is defined by
\begin{equation}
d_r(y)=\min_{g\in \mathcal{G}}(y-rg)^2 .
\end{equation}
Let $y_i=|x_i|/s_0$ and define the scale ratio $r=s/s_0$. Then minimizing $\mathcal{L}(s;\mathbf{x})$ over FP8 scales $s$ is equivalent to minimizing
\begin{equation}
\bar{\mathcal{L}}(r;\mathbf{y})
=
\sum_{i=0}^{15} d_r(y_i)
\end{equation}
over the corresponding FP8 ratios $r$.
Furthermore, suppose that
\begin{equation}
s_0=\frac{m}{8}2^a,
\qquad
m\in\{8,\dots,15\}.
\end{equation}
Then the normalized coordinates satisfy
\begin{equation}
\max_i y_i = 6\alpha,
\qquad
1 \le \alpha < \rho_m,
\end{equation}
where $\rho_m=\frac{m+1}{m}$.

\begin{table*}[ht]
\renewcommand{\arraystretch}{1.3}
\centering
\begin{tabular}{cccccc}
\toprule
$r$ & $\{u_j\}$ & $\{\lambda_j\}$ & $B$ & $A$ & $A+15B$ \\
\midrule
$\frac{7}{9}$ & $\{\frac{5}{6},\frac{10}{9},\frac{14}{9}\}$ & $\{\frac{1}{4},\frac{1}{4},\frac{1}{2}\}$ & $\frac{143}{1728}$ & $-\frac{451}{324}$ & $-\frac{781}{5184}$ \\
$\frac{10}{13}$ & $\{\frac{11}{13},\frac{14}{13},\frac{20}{13}\}$ & $\{\frac{1}{4},\frac{1}{4},\frac{1}{2}\}$ & $\frac{18}{169}$ & $-\frac{277}{169}$ & $-\frac{7}{169}$ \\
$\frac{11}{14}$ & $\{\frac{6}{7},\frac{15}{14},\frac{11}{7}\}$ & $\{\frac{1}{4},\frac{1}{4},\frac{1}{2}\}$ & $\frac{9}{98}$ & $-\frac{271}{196}$ & $-\frac{1}{196}$ \\
\bottomrule
\end{tabular}
\caption{
Convex certificates used to exclude the remaining FP8 E4M3 scale ratios in $(3/4,4/5)$. For each candidate ratio $r$, the listed FP8 ratios $\{u_j\}$ and weights $\{\lambda_j\}$ satisfy $\sum_j\lambda_j=1$. The positive margin $A_r-15B_r$ proves that the weighted average loss over $\{u_j\}$ is strictly smaller than the loss at $r$ for any block of size $16$.
}
\label{tab:lower_bound_convex_certificates}
\end{table*}

\paragraph{Proof sketch.}
For each candidate ratio $r<\tfrac54$, we find another valid ratio $r'$ satisfying
\begin{equation}
r<r'\le \tfrac{11}{4},
\end{equation}
and show that
\begin{equation}
\bar{\mathcal{L}}(r';\mathbf{y})
\le
\bar{\mathcal{L}}(r;\mathbf{y})
\end{equation}
for all admissible normalized coordinates $\mathbf{y}$. The worst-case increase of
\begin{equation}
\bar{\mathcal{L}}(r';\mathbf{y})
-
\bar{\mathcal{L}}(r;\mathbf{y})
\end{equation}
is obtained by combining the largest possible contribution from the maximal coordinate with the largest possible contribution from each of the remaining $15$ coordinates. Therefore, it suffices to verify
\begin{align}
A_{r',r}
&=
\sup_{6\le y\le 6\rho_m}
\left(d_{r'}(y)-d_r(y)\right), \notag \\
B_{r',r}
&=
\sup_{0\le y\le 6\rho_m}
\left(d_{r'}(y)-d_r(y)\right), \notag \\
A_{r',r} &+15B_{r',r} \le 0.
\end{align}

\paragraph{Excluding $r\leq \tfrac{3}{4}$.}
We first consider ratios satisfying $r\le \tfrac34$, for which we choose $r'=2r$. The grids induced by $r$ and $2r$ share most reconstruction points, while the grid associated with $r$ additionally contains the intermediate points $\tfrac12 r$ and $\tfrac32 r$. Therefore,
\begin{equation}
d_{2r}(y)-d_r(y)
\le
\frac{r^2}{4},
\qquad
0\le y\leq 6\rho_m,
\end{equation}
where equality is attained at $y=\tfrac12 r$ and $y=\tfrac32 r$. Consequently,
\begin{equation}
B_{2r,r}
=
\sup_{0\le y<6\rho_m}
\left(d_{2r}(y)-d_r(y)\right)
=
\frac{r^2}{4}.
\end{equation}
Therefore, replacing $r$ with $2r$ increases the total loss on the remaining $15$ coordinates by at most $15r^2/4$.

Now consider the maximal coordinate $y_{\max}=6\alpha$. Since $r\le \tfrac34$ and $\alpha\ge 1$, the maximal reconstruction value achievable under scale $r$ equals $6r$, which gives
\begin{equation}
d_r(6\alpha)=(6\alpha-6r)^2.
\end{equation}
Under scale $2r$, selecting the FP4 grid point $4$ yields reconstruction value $8r$, and hence
\begin{equation}
d_{2r}(6\alpha)\le (6\alpha-8r)^2.
\end{equation}
Therefore,
\begin{align}
d_{2r}(6\alpha)-d_r(6\alpha)
&\le
(6\alpha-8r)^2-(6\alpha-6r)^2 \notag \\
&=
-24\alpha r+28r^2 .
\end{align}
Using $\alpha\ge 1$ and $r\le \tfrac34$ yields
\begin{align}
A_{2r,r}
&=
\sup_{6\le y\le 6\rho_m}
\left(d_{2r}(y)-d_r(y)\right) \notag \\
&\le
-24\alpha r+28r^2 \notag \\
&\le
-24r+28r^2 .
\end{align}
Combining this bound with
\begin{equation}
B_{2r,r}=\frac{r^2}{4}
\end{equation}
gives
\begin{align}
A_{2r,r}+15B_{2r,r}
&\le
-24r+28r^2+\frac{15r^2}{4} \notag \\
&=
-r\left(24-\frac{127}{4}r\right).
\end{align}
Since $r\le \tfrac34$,
\begin{equation}
24-\frac{127}{4}r
\ge
24-\frac{127}{4}\cdot\frac34
=
\frac{3}{16}
>0.
\end{equation}
Hence,
\begin{equation}
A_{2r,r}+15B_{2r,r}<0,
\end{equation}
which implies
\begin{equation}
\bar{\mathcal{L}}(2r;\mathbf{y})
<
\bar{\mathcal{L}}(r;\mathbf{y}).
\end{equation}
Therefore, no ratio $r\le \tfrac34$ can be optimal.

\paragraph{Excluding $r\in (3/4, 4/5)$}
It remains to consider FP8 ratios in the interval
$(3/4,4/5)$. Since a positive normal E4M3 value has significand
$m/8$ with $m\in\{8,\dots,15\}$, every ratio between two such FP8
values has the form
\begin{equation}
r=\frac{j}{m}2^{\Delta a},
\qquad
j,m\in\{8,\dots,15\}.
\end{equation}
A finite enumeration shows that the only legal ratios in
$(3/4,4/5)$ are
\begin{equation}
\frac79,\qquad \frac{10}{13},\qquad \frac{11}{14}. \notag
\end{equation}

We exclude the remaining cases using convex certificates. For a fixed candidate ratio $r$, consider valid FP8 ratios $u_j$ and nonnegative weights $\lambda_j$ satisfying $\sum_j\lambda_j=1$, and define
\begin{equation}
\Delta_{\{u_j\},r}(y)=\sum_j \lambda_j d_{u_j}(y)-d_r(y).
\end{equation}
This construction extends the previous argument by allowing comparison against a convex combination of multiple ratios instead of a single ratio $r'$. Define
\begin{align}
A_{\{u_j\},r}&=\sup_{6\le y\le 6\rho_m}\Delta_{\{u_j\},r}(y), \notag \\
B_{\{u_j\},r}&=\sup_{0\le y\le 6\rho_m}\Delta_{\{u_j\},r}(y).
\end{align}
If
\begin{equation}
A_{\{u_j\},r}+15B_{\{u_j\},r}\leq 0,
\end{equation}
then for any block containing a maximal coordinate,
\begin{align}
& \sum_j\lambda_j\bar{\mathcal{L}}(u_j;\mathbf{y})-\bar{\mathcal{L}}(r;\mathbf{y}) \\
&= \sum_{i=0}^{15}\Delta_{\{u_j\},r}(y_i) \notag \\
&\le A_{\{u_j\},r}+15B_{\{u_j\},r} \\
& \leq 0.
\end{align}
Therefore, at least one ratio $u_j$ achieves strictly smaller loss than $r$, implying that $r$ cannot be optimal.

The required certificates are summarized in Table~\ref{tab:lower_bound_convex_certificates}. All listed certificates satisfy
\begin{equation}
A_{\{u_j\},r}+15B_{\{u_j\},r}\leq 0.
\end{equation}
Since $\sum_j\lambda_j=1$, the quadratic $y^2$ terms cancel in $\Delta_{\{u_j\},r}(y)$. Therefore, $\Delta_{\{u_j\},r}(y)$ is piecewise affine in $y$, and all extrema are attained at breakpoints. Consequently, both $A_{\{u_j\},r}$ and $B_{\{u_j\},r}$ can be computed by enumerating finitely many breakpoints. The coefficients $\lambda_j$ are obtained either by exhaustive search or by solving a linear program.

\paragraph{Summary}
In summary, for every ratio satisfying $r\le \tfrac45$, there exists a valid FP8 ratio $r'\le \tfrac{11}{4}$, or more generally a convex combination of valid FP8 ratios $\{u_j\}$ with $u_j\le \tfrac{11}{4}$, such that
\begin{equation}
\bar{\mathcal{L}}(r';\mathbf{y})
\le
\bar{\mathcal{L}}(r;\mathbf{y}),
\end{equation}
or
\begin{equation}
\min_{j}\bar{\mathcal{L}}(u_j;\mathbf{y})\le \sum_j\lambda_j\bar{\mathcal{L}}(u_j;\mathbf{y})
\le
\bar{\mathcal{L}}(r;\mathbf{y}),
\end{equation}
for all admissible normalized coordinates $\mathbf{y}$. Therefore, no ratio $r\le \tfrac45$ can be optimal.
\end{proof}

\begin{table*}[ht]
\small
\centering
\begin{tabular}{ccc}
\toprule
\textbf{Dataset} & \textbf{Size} & \textbf{Evaluation Domain} \\
\midrule
MMLU-Pro & 12,032 & Broad domain knowledge and reasoning. \\
GSM8K & 1,319 & Grade school mathematical reasoning. \\
IFEval & 541 & Verifiable instruction following. \\
HellaSwag & 10,042 & Commonsense sentence completion. \\
WinoGrande & 1,267 & Winograd style commonsense reasoning. \\
\bottomrule
\end{tabular}
\caption{Dataset size and domain for MMLU Pro, GSM8K, IFEval, HellaSwag, and WinoGrande.}
\label{tab:test_datasets}
\end{table*}

\section{Optimal FP32 Scale for FP4 Quantization} \label{appendix:optimal-scale-for-fp4}

\paragraph{Optimal FP32 Block Scale}
The WMSE objective in Eq.~\eqref{eq:weighted_quantization_loss} can be equivalently written as
\begin{equation}
\mathcal{L}(s; \mathbf{x}, \mathbf{w}) =
\sum_{i=0}^{n-1} w_i \left( x_i - \left\lfloor x_i/s \right\rceil_{\mathrm{FP4}} \cdot s \right)^2.
\end{equation}
Each term $w_i \left( x_i - \left\lfloor x_i/s \right\rceil_{\mathrm{FP4}} \cdot s \right)^2$ constitutes a piecewise quadratic function of $s$, with 15 segments corresponding to the breakpoints at which $x_i / s$ aligns with the midpoints between consecutive FP4 representable values. Consequently, $\mathcal{L}(s; \mathbf{x}, \mathbf{w})$ is a piecewise quadratic function of $s$, with breakpoints given by the union of all individual breakpoints, yielding at most $14n+1$ segments. The exact global minimum of $\mathcal{L}(s; \mathbf{x}, \mathbf{w})$ can be computed efficiently by enumerating all segments, determining the quadratic coefficients within each, and evaluating the minimum of each quadratic segment.

\paragraph{FP8-quantized Optimal FP32 Scale}
The FP8-quantized optimal FP32 scale $s_{\mathrm{fp8}}$ is obtained by quantizing the optimal FP32 scale $s$ with the global scale $S$, where $s_{\mathrm{fp8}}$ is selected between $\lfloor s / S \rfloor_\mathrm{FP8}$ and $\lceil s / S \rceil_\mathrm{FP8}$ according to the minimization of $\mathcal{L}(s_{\mathrm{fp8}}\cdot S; \mathbf{x}, \mathbf{w})$.

\section{Additional Implementation Details} \label{appendix:implementation_details}

\subsection{Calibration Details}

\paragraph{Calibration} We follow PV-Tuning~\citep{pv-tuning} and sample 128 sequences of length 2048 from the 1B samples released in the official RedPajama dataset as calibration samples. Calibration is performed only once using a fixed seed of 0.

\paragraph{Quantization} For activation quantization, KV cache quantization, and query states quantization in NVFP4, we follow the practices presented in LLM-Compressor\footnote{\href{https://docs.vllm.ai/projects/llm-compressor/en/latest/examples/quantization_w4a4_fp4/\#2-prepare-calibration-data}{LLM-Compressor W4A4 FP4 quantization example}} and TensorRT\footnote{\href{https://docs.nvidia.com/deeplearning/tensorrt/latest/inference-library/work-quantized-types.html\#dynamic-double-quantization}{TensorRT dynamic double quantization guide}}, employing a static global scale. The global scale is determined during the calibration stage and remains fixed throughout the inference stage.

\paragraph{GPTQ} 
For GPTQ~\citep{gptq}, we employ LDLQ~\citep{quip}, which has been shown to be algorithmically equivalent to GPTQ. The damping factor is set to $0.01$ times the mean of the Hessian diagonal. Following MR-GPTQ~\citep{mr_gptq}, we apply static reordering: the block scales is determined prior to GPTQ, and the input channels for GPTQ quantization are ordered according to the Hessian diagonal in descending order.

\subsection{Evaluation Details}

For MMLU Pro, GSM8K, IFEval, HellaSwag, and WinoGrande, we use lm-evaluation-harness~\citep{eval-harness} for evaluation, corresponding to the tasks \texttt{mmlu\_pro\_llama}, \texttt{gsm8k\_llama}, \texttt{ifeval}, \texttt{hellaswag}, and \texttt{winogrande}, respectively. Among these, \texttt{apply\_chat\_template} is enabled for MMLU Pro, GSM8K, and IFEval. The size and evaluation domain of the five benchmarks are summarized in Table~\ref{tab:test_datasets}.

\subsection{Operator Benchmark Details}

We conduct operator-level benchmarks using PyTorch 2.9.0, Triton 3.5.0, and vLLM 0.13.0 on NVIDIA RTX PRO 6000 Blackwell GPUs. Since the default NVFP4 operators in vLLM employ a swizzle layout, the batch dimension is padded to multiples of 128. Consequently, latency results for batch sizes smaller than 128 may not accurately.

\section{Full Result} \label{appendix:full_result}

Table~\ref{tab:kernel_latency} presents the latency comparison and relative latency ratios under different batch sizes with hidden dimension 8192 on NVIDIA RTX PRO 6000 Blackwell GPUs.
Figure~\ref{fig:mse_and_wmse_on_each_layer_QKV} presents the MSE and WMSE gap (\%) of different initialization methods relative to the optimal FP32 block scale for key cache, value cache, and query state quantization on all layers of Llama-3.1-8B-Instruct.
Table~\ref{tab:main_result_llama_3_2_3B} presents the results on Llama-3.2-3B-Instruct, while Table~\ref{tab:main_result_qwen_3_4B_8B} presents the results on Qwen3-4B and Qwen3-8B.

\begin{table*}
\small
\centering
\begin{tabular}{llrrrrrrrrr}
\toprule
\textbf{Method} & \textbf{Metric} & 128 & 256 & 512 & 1024 & 2048 & 4096 & 8192 & 16384 & 32768 \\
\midrule
vllm default & Latency(ms) & 0.010 & 0.012 & 0.016 & 0.025 & 0.041 & 0.072 & 0.135 & 0.258 & 0.502 \\
ScaleSweep$_\mathrm{MSE}$ & Latency(ms) & 0.010 & 0.013 & 0.017 & 0.026 & 0.042 & 0.072 & 0.134 & 0.260 & 0.502 \\
ScaleSweep$_\mathrm{MSE}$ & Rel. Latency & 0.97$\times$ & 1.08$\times$ & 1.03$\times$ & 1.03$\times$ & 1.01$\times$ & 0.99$\times$ & 1.00$\times$ & 1.01$\times$ & 1.00$\times$ \\
ScaleSweep & Latency(ms) & 0.012 & 0.016 & 0.021 & 0.032 & 0.055 & 0.099 & 0.185 & 0.357 & 0.689 \\
ScaleSweep & Rel. Latency & 1.23$\times$ & 1.35$\times$ & 1.26$\times$ & 1.29$\times$ & 1.33$\times$ & 1.36$\times$ & 1.38$\times$ & 1.38$\times$ & 1.37$\times$ \\
\bottomrule
\end{tabular}
\caption{Comparison of latency and relative latency ratio of the NVFP4 quantization operator across different batch sizes with hidden dimension 8192 on NVIDIA RTX PRO 6000 Blackwell GPUs. ``Rel. Latency'' refers to relative latency ratio, which denotes the latency ratio relative to the default vLLM operator.}
\label{tab:kernel_latency}
\end{table*}

\begin{figure*}[t]
  \includegraphics[width=\textwidth]{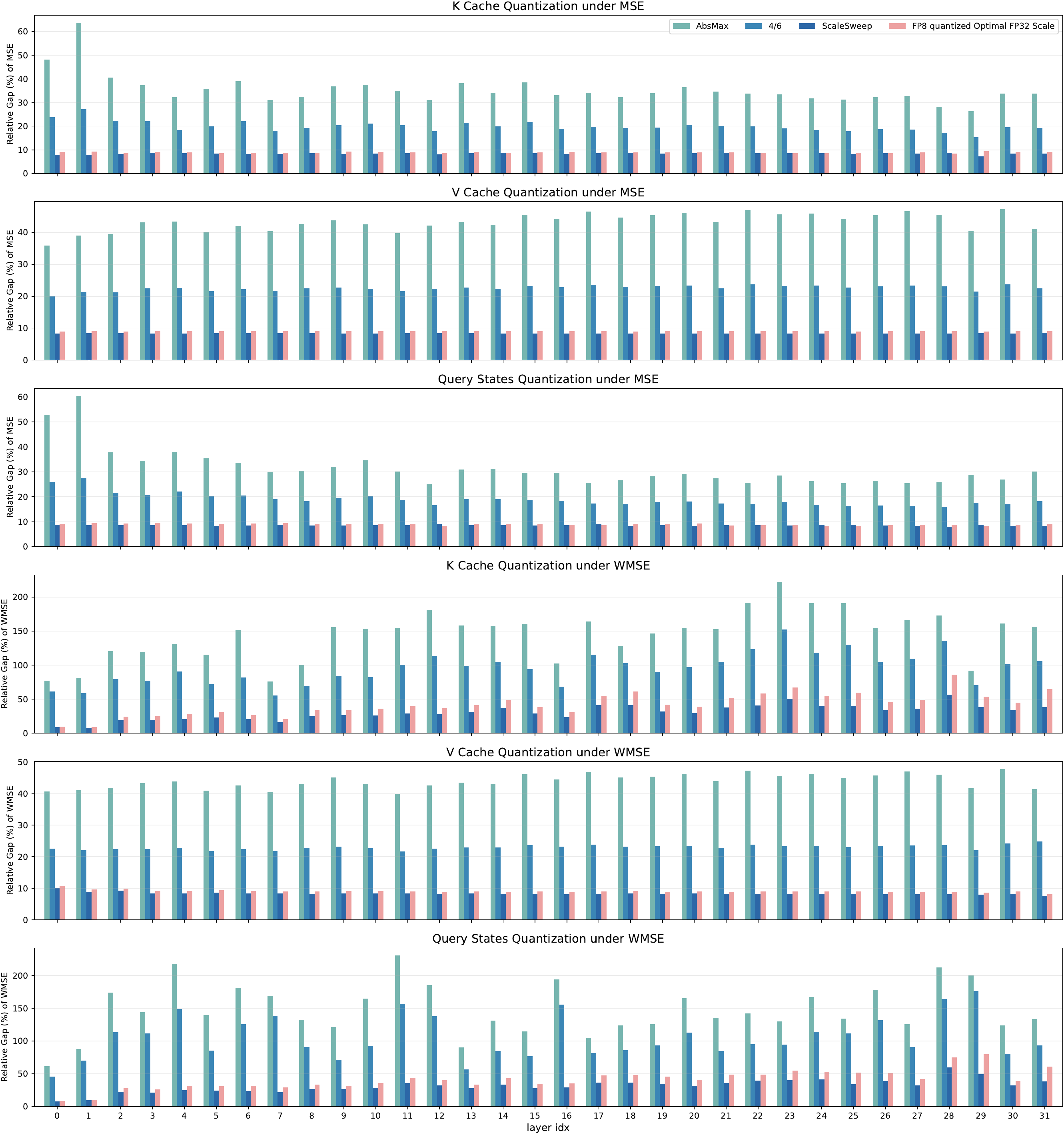}
  \caption{Performance gap (\%) of ScaleSweep, initialization baselines, and FP8-quantized optimal FP32 scales relative to the optimal FP32 block scale in terms of MSE and WMSE for KV cache quantization and query states quantization on Llama-3.1-8B-Instruct.}
  \label{fig:mse_and_wmse_on_each_layer_QKV}
\end{figure*}

\begin{table*}[th]
\small
\centering
\begin{tabular}{ccc@{\hskip 5pt}c@{\hskip 4pt}cccccc@{\hskip 3pt}c}
\toprule
\textbf{Setting} & \textbf{Method} & \textbf{Init Method} & \textbf{MMLU Pro} & \textbf{GSM8k} & \textbf{IFEval} & \textbf{HellaS} & \textbf{WinoG} & \textbf{Avg} & \textbf{Recovery (\%)} \\
\midrule
\multicolumn{10}{c}{\textbf{Llama-3.2-3B-Instruct}} \\
\midrule- & BF16 & - & 37.30 & 78.54 & 70.24 & 71.75 & 68.51 & 65.27 & 100 \\
\cmidrule(l){1-10}
\multirow[c]{8}{*}{WA} & \multirow[c]{4}{*}{RTN} & AbsMax & 33.43 & 71.49 & 63.96 & 70.30 & 68.27 & 61.49 & 94.21 \\
 &  & 4/6 & 34.10 & 71.95 & 67.28 & 70.11 & 65.51 & 61.79 & 94.67 \\
 &  & \cellcolor{blue!15}ScaleSweep$_\mathrm{MSE}$ & \cellcolor{blue!15}33.96 & \cellcolor{blue!15}\textbf{73.46} & \cellcolor{blue!15}68.02 & \cellcolor{blue!15}\textbf{70.39} & \cellcolor{blue!15}67.09 & \cellcolor{blue!15}62.59 & \cellcolor{blue!15}95.89 \\
 &  & \cellcolor{blue!15}ScaleSweep & \cellcolor{blue!15}\textbf{34.12} & \cellcolor{blue!15}73.16 & \cellcolor{blue!15}\textbf{68.76} & \cellcolor{blue!15}70.18 & \cellcolor{blue!15}\textbf{68.98} & \cellcolor{blue!15}\textbf{63.04} & \cellcolor{blue!15}\textbf{96.58} \\
\cmidrule(l){2-10}
 & \multirow[c]{4}{*}{GPTQ} & AbsMax & 33.29 & 71.57 & 65.43 & 69.85 & \textbf{68.75} & 61.78 & 94.65 \\
 &  & 4/6 & 34.23 & 72.78 & 69.69 & \textbf{70.53} & 66.06 & 62.66 & 96.00 \\
 &  & \cellcolor{blue!15}ScaleSweep$_\mathrm{MSE}$ & \cellcolor{blue!15}\textbf{35.09} & \cellcolor{blue!15}\textbf{74.68} & \cellcolor{blue!15}65.62 & \cellcolor{blue!15}70.26 & \cellcolor{blue!15}\textbf{68.75} & \cellcolor{blue!15}62.88 & \cellcolor{blue!15}96.34 \\
 &  & \cellcolor{blue!15}ScaleSweep & \cellcolor{blue!15}34.97 & \cellcolor{blue!15}73.69 & \cellcolor{blue!15}\textbf{70.61} & \cellcolor{blue!15}70.37 & \cellcolor{blue!15}66.46 & \cellcolor{blue!15}\textbf{63.22} & \cellcolor{blue!15}\textbf{96.86} \\
\cmidrule(l){1-10}
\multirow[c]{8}{*}{WAKV} & \multirow[c]{4}{*}{RTN} & AbsMax & 30.93 & 67.17 & 66.36 & 69.74 & \textbf{68.03} & 60.45 & 92.61 \\
 &  & 4/6 & 31.99 & 69.83 & \textbf{67.84} & 69.63 & 66.69 & 61.19 & 93.76 \\
 &  & \cellcolor{blue!15}ScaleSweep$_\mathrm{MSE}$ & \cellcolor{blue!15}33.08 & \cellcolor{blue!15}\textbf{72.33} & \cellcolor{blue!15}67.47 & \cellcolor{blue!15}69.79 & \cellcolor{blue!15}66.30 & \cellcolor{blue!15}61.79 & \cellcolor{blue!15}94.67 \\
 &  & \cellcolor{blue!15}ScaleSweep & \cellcolor{blue!15}\textbf{33.10} & \cellcolor{blue!15}\textbf{72.33} & \cellcolor{blue!15}\textbf{67.84} & \cellcolor{blue!15}\textbf{70.08} & \cellcolor{blue!15}67.80 & \cellcolor{blue!15}\textbf{62.23} & \cellcolor{blue!15}\textbf{95.34} \\
\cmidrule(l){2-10}
 & \multirow[c]{4}{*}{GPTQ} & AbsMax & 31.26 & 67.63 & 66.73 & 69.81 & 66.85 & 60.45 & 92.62 \\
 &  & 4/6 & 32.41 & 71.11 & 66.17 & 69.50 & \textbf{67.40} & 61.32 & 93.95 \\
 &  & \cellcolor{blue!15}ScaleSweep$_\mathrm{MSE}$ & \cellcolor{blue!15}33.44 & \cellcolor{blue!15}\textbf{72.33} & \cellcolor{blue!15}65.62 & \cellcolor{blue!15}\textbf{70.22} & \cellcolor{blue!15}66.93 & \cellcolor{blue!15}61.71 & \cellcolor{blue!15}94.54 \\
 &  & \cellcolor{blue!15}ScaleSweep & \cellcolor{blue!15}\textbf{33.75} & \cellcolor{blue!15}71.80 & \cellcolor{blue!15}\textbf{67.65} & \cellcolor{blue!15}70.07 & \cellcolor{blue!15}67.25 & \cellcolor{blue!15}\textbf{62.10} & \cellcolor{blue!15}\textbf{95.15} \\
\cmidrule(l){1-10}
\multirow[c]{8}{*}{WAKVQ} & \multirow[c]{4}{*}{RTN} & AbsMax & 28.00 & 61.87 & 63.96 & 68.71 & 65.43 & 57.59 & 88.24 \\
 &  & 4/6 & 30.44 & 67.32 & 60.44 & 69.46 & \textbf{67.32} & 59.00 & 90.39 \\
 &  & \cellcolor{blue!15}ScaleSweep$_\mathrm{MSE}$ & \cellcolor{blue!15}31.33 & \cellcolor{blue!15}68.31 & \cellcolor{blue!15}66.17 & \cellcolor{blue!15}69.13 & \cellcolor{blue!15}66.69 & \cellcolor{blue!15}60.33 & \cellcolor{blue!15}92.43 \\
 &  & \cellcolor{blue!15}ScaleSweep & \cellcolor{blue!15}\textbf{31.93} & \cellcolor{blue!15}\textbf{69.90} & \cellcolor{blue!15}\textbf{66.91} & \cellcolor{blue!15}\textbf{69.73} & \cellcolor{blue!15}66.38 & \cellcolor{blue!15}\textbf{60.97} & \cellcolor{blue!15}\textbf{93.41} \\
\cmidrule(l){2-10}
 & \multirow[c]{4}{*}{GPTQ} & AbsMax & 29.81 & 62.62 & 61.74 & 69.05 & 65.35 & 57.71 & 88.43 \\
 &  & 4/6 & 30.73 & 66.11 & 65.80 & 69.01 & 64.80 & 59.29 & 90.84 \\
 &  & \cellcolor{blue!15}ScaleSweep$_\mathrm{MSE}$ & \cellcolor{blue!15}32.21 & \cellcolor{blue!15}68.76 & \cellcolor{blue!15}62.66 & \cellcolor{blue!15}69.25 & \cellcolor{blue!15}65.82 & \cellcolor{blue!15}59.74 & \cellcolor{blue!15}91.53 \\
 &  & \cellcolor{blue!15}ScaleSweep & \cellcolor{blue!15}\textbf{33.02} & \cellcolor{blue!15}\textbf{71.11} & \cellcolor{blue!15}\textbf{65.99} & \cellcolor{blue!15}\textbf{69.43} & \cellcolor{blue!15}\textbf{66.30} & \cellcolor{blue!15}\textbf{61.17} & \cellcolor{blue!15}\textbf{93.72} \\
\bottomrule
\end{tabular}
\caption{Comparison results of ScaleSweep and baseline initialization methods under RTN and GPTQ across different settings on Llama-3.2-3B-Instruct.} \label{tab:main_result_llama_3_2_3B}
\end{table*}

\begin{table*}
\small
\centering
\begin{tabular}{ccc@{\hskip 5pt}c@{\hskip 4pt}cccccc@{\hskip 3pt}c}
\toprule
\textbf{Setting} & \textbf{Method} & \textbf{Init Method} & \textbf{MMLU Pro} & \textbf{GSM8k} & \textbf{IFEval} & \textbf{HellaS} & \textbf{WinoG} & \textbf{Avg} & \textbf{Recovery (\%)} \\
\midrule
\multicolumn{10}{c}{\textbf{Qwen3-8B}} \\
\midrule- & BF16 & - & 61.99 & 91.28 & 81.70 & 74.91 & 68.11 & 75.60 & 100 \\
\cmidrule(l){1-10}
\multirow[c]{8}{*}{WA} & \multirow[c]{4}{*}{RTN} & AbsMax & 58.78 & 89.31 & 81.33 & 73.20 & 68.19 & 74.16 & 98.10 \\
 &  & 4/6 & 59.38 & 89.61 & 81.15 & 73.37 & 68.03 & 74.31 & 98.30 \\
 &  & \cellcolor{blue!15}ScaleSweep$_\mathrm{MSE}$ & \cellcolor{blue!15}\textbf{59.67} & \cellcolor{blue!15}89.23 & \cellcolor{blue!15}82.62 & \cellcolor{blue!15}73.66 & \cellcolor{blue!15}67.64 & \cellcolor{blue!15}74.57 & \cellcolor{blue!15}98.63 \\
 &  & \cellcolor{blue!15}ScaleSweep & \cellcolor{blue!15}59.57 & \cellcolor{blue!15}\textbf{90.37} & \cellcolor{blue!15}\textbf{83.18} & \cellcolor{blue!15}\textbf{73.68} & \cellcolor{blue!15}\textbf{69.30} & \cellcolor{blue!15}\textbf{75.22} & \cellcolor{blue!15}\textbf{99.50} \\
\cmidrule(l){2-10}
 & \multirow[c]{4}{*}{GPTQ} & AbsMax & \textbf{58.29} & 88.78 & 80.41 & 73.50 & 66.77 & 73.55 & 97.29 \\
 &  & 4/6 & 52.24 & 89.23 & 81.52 & 73.31 & 67.25 & 72.71 & 96.18 \\
 &  & \cellcolor{blue!15}ScaleSweep$_\mathrm{MSE}$ & \cellcolor{blue!15}56.91 & \cellcolor{blue!15}\textbf{89.99} & \cellcolor{blue!15}\textbf{82.44} & \cellcolor{blue!15}\textbf{73.97} & \cellcolor{blue!15}\textbf{67.96} & \cellcolor{blue!15}\textbf{74.25} & \cellcolor{blue!15}\textbf{98.22} \\
 &  & \cellcolor{blue!15}ScaleSweep & \cellcolor{blue!15}57.95 & \cellcolor{blue!15}89.31 & \cellcolor{blue!15}82.07 & \cellcolor{blue!15}73.41 & \cellcolor{blue!15}67.80 & \cellcolor{blue!15}74.11 & \cellcolor{blue!15}98.03 \\
\cmidrule(l){1-10}
\multirow[c]{8}{*}{WAKV} & \multirow[c]{4}{*}{RTN} & AbsMax & 56.48 & 86.81 & 80.59 & 72.34 & \textbf{68.59} & 72.96 & 96.51 \\
 &  & 4/6 & 56.09 & 89.46 & 80.59 & 73.54 & 66.14 & 73.17 & 96.78 \\
 &  & \cellcolor{blue!15}ScaleSweep$_\mathrm{MSE}$ & \cellcolor{blue!15}57.50 & \cellcolor{blue!15}89.16 & \cellcolor{blue!15}81.70 & \cellcolor{blue!15}73.50 & \cellcolor{blue!15}65.98 & \cellcolor{blue!15}73.57 & \cellcolor{blue!15}97.31 \\
 &  & \cellcolor{blue!15}ScaleSweep & \cellcolor{blue!15}\textbf{58.05} & \cellcolor{blue!15}\textbf{89.61} & \cellcolor{blue!15}\textbf{83.73} & \cellcolor{blue!15}\textbf{73.79} & \cellcolor{blue!15}65.90 & \cellcolor{blue!15}\textbf{74.22} & \cellcolor{blue!15}\textbf{98.17} \\
\cmidrule(l){2-10}
 & \multirow[c]{4}{*}{GPTQ} & AbsMax & 54.90 & 88.55 & \textbf{83.92} & 72.97 & \textbf{67.40} & \textbf{73.55} & \textbf{97.29} \\
 &  & 4/6 & 47.91 & 88.55 & 81.33 & 73.04 & 67.25 & 71.62 & 94.73 \\
 &  & \cellcolor{blue!15}ScaleSweep$_\mathrm{MSE}$ & \cellcolor{blue!15}\textbf{55.49} & \cellcolor{blue!15}88.40 & \cellcolor{blue!15}79.48 & \cellcolor{blue!15}73.56 & \cellcolor{blue!15}66.30 & \cellcolor{blue!15}72.65 & \cellcolor{blue!15}96.09 \\
 &  & \cellcolor{blue!15}ScaleSweep & \cellcolor{blue!15}51.95 & \cellcolor{blue!15}\textbf{89.31} & \cellcolor{blue!15}80.96 & \cellcolor{blue!15}\textbf{73.88} & \cellcolor{blue!15}66.93 & \cellcolor{blue!15}72.61 & \cellcolor{blue!15}96.04 \\
\cmidrule(l){1-10}
\multirow[c]{8}{*}{WAKVQ} & \multirow[c]{4}{*}{RTN} & AbsMax & 52.98 & 87.11 & \textbf{82.62} & 71.83 & \textbf{67.48} & 72.40 & 95.77 \\
 &  & 4/6 & 53.71 & 87.64 & 80.41 & \textbf{72.92} & 66.46 & 72.23 & 95.54 \\
 &  & \cellcolor{blue!15}ScaleSweep$_\mathrm{MSE}$ & \cellcolor{blue!15}56.81 & \cellcolor{blue!15}88.70 & \cellcolor{blue!15}81.70 & \cellcolor{blue!15}72.84 & \cellcolor{blue!15}66.69 & \cellcolor{blue!15}\textbf{73.35} & \cellcolor{blue!15}\textbf{97.02} \\
 &  & \cellcolor{blue!15}ScaleSweep & \cellcolor{blue!15}\textbf{57.34} & \cellcolor{blue!15}\textbf{88.86} & \cellcolor{blue!15}80.96 & \cellcolor{blue!15}72.66 & \cellcolor{blue!15}65.98 & \cellcolor{blue!15}73.16 & \cellcolor{blue!15}96.77 \\
\cmidrule(l){2-10}
 & \multirow[c]{4}{*}{GPTQ} & AbsMax & 51.19 & 88.32 & \textbf{82.07} & 72.58 & \textbf{67.56} & 72.34 & 95.69 \\
 &  & 4/6 & 48.09 & 88.25 & 80.59 & 72.78 & 67.17 & 71.38 & 94.41 \\
 &  & \cellcolor{blue!15}ScaleSweep$_\mathrm{MSE}$ & \cellcolor{blue!15}53.55 & \cellcolor{blue!15}\textbf{88.48} & \cellcolor{blue!15}80.22 & \cellcolor{blue!15}73.18 & \cellcolor{blue!15}65.98 & \cellcolor{blue!15}72.28 & \cellcolor{blue!15}95.61 \\
 &  & \cellcolor{blue!15}ScaleSweep & \cellcolor{blue!15}\textbf{56.08} & \cellcolor{blue!15}\textbf{88.48} & \cellcolor{blue!15}80.59 & \cellcolor{blue!15}\textbf{73.20} & \cellcolor{blue!15}66.69 & \cellcolor{blue!15}\textbf{73.01} & \cellcolor{blue!15}\textbf{96.57} \\
\midrule
\multicolumn{10}{c}{\textbf{Qwen3-4B}} \\
\midrule- & BF16 & - & 57.86 & 88.25 & 79.85 & 68.32 & 66.06 & 72.07 & 100 \\
\cmidrule(l){1-10}
\multirow[c]{8}{*}{WA} & \multirow[c]{4}{*}{RTN} & AbsMax & 52.86 & 86.58 & \textbf{79.67} & 65.77 & 62.51 & 69.48 & 96.40 \\
 &  & 4/6 & 54.21 & 86.73 & 78.56 & 66.02 & \textbf{63.85} & \textbf{69.87} & \textbf{96.95} \\
 &  & \cellcolor{blue!15}ScaleSweep$_\mathrm{MSE}$ & \cellcolor{blue!15}53.37 & \cellcolor{blue!15}86.05 & \cellcolor{blue!15}78.93 & \cellcolor{blue!15}\textbf{66.67} & \cellcolor{blue!15}63.30 & \cellcolor{blue!15}69.66 & \cellcolor{blue!15}96.66 \\
 &  & \cellcolor{blue!15}ScaleSweep & \cellcolor{blue!15}\textbf{54.69} & \cellcolor{blue!15}\textbf{88.02} & \cellcolor{blue!15}76.52 & \cellcolor{blue!15}66.61 & \cellcolor{blue!15}61.25 & \cellcolor{blue!15}69.42 & \cellcolor{blue!15}96.32 \\
\cmidrule(l){2-10}
 & \multirow[c]{4}{*}{GPTQ} & AbsMax & 53.74 & 85.82 & 77.82 & 66.48 & 64.09 & 69.59 & 96.56 \\
 &  & 4/6 & 54.16 & 86.88 & 77.82 & 66.37 & 65.43 & 70.13 & 97.31 \\
 &  & \cellcolor{blue!15}ScaleSweep$_\mathrm{MSE}$ & \cellcolor{blue!15}54.75 & \cellcolor{blue!15}86.28 & \cellcolor{blue!15}\textbf{79.11} & \cellcolor{blue!15}66.34 & \cellcolor{blue!15}\textbf{65.51} & \cellcolor{blue!15}\textbf{70.40} & \cellcolor{blue!15}\textbf{97.68} \\
 &  & \cellcolor{blue!15}ScaleSweep & \cellcolor{blue!15}\textbf{54.95} & \cellcolor{blue!15}\textbf{87.04} & \cellcolor{blue!15}77.26 & \cellcolor{blue!15}\textbf{66.98} & \cellcolor{blue!15}62.98 & \cellcolor{blue!15}69.84 & \cellcolor{blue!15}96.91 \\
\cmidrule(l){1-10}
\multirow[c]{8}{*}{WAKV} & \multirow[c]{4}{*}{RTN} & AbsMax & 49.33 & 84.23 & 76.52 & 65.12 & 62.43 & 67.53 & 93.70 \\
 &  & 4/6 & 51.38 & 83.78 & \textbf{79.85} & 65.46 & 61.72 & 68.44 & 94.96 \\
 &  & \cellcolor{blue!15}ScaleSweep$_\mathrm{MSE}$ & \cellcolor{blue!15}51.02 & \cellcolor{blue!15}85.37 & \cellcolor{blue!15}78.19 & \cellcolor{blue!15}65.64 & \cellcolor{blue!15}62.12 & \cellcolor{blue!15}68.47 & \cellcolor{blue!15}95.00 \\
 &  & \cellcolor{blue!15}ScaleSweep & \cellcolor{blue!15}\textbf{52.79} & \cellcolor{blue!15}\textbf{87.11} & \cellcolor{blue!15}77.45 & \cellcolor{blue!15}\textbf{66.32} & \cellcolor{blue!15}\textbf{62.51} & \cellcolor{blue!15}\textbf{69.24} & \cellcolor{blue!15}\textbf{96.07} \\
\cmidrule(l){2-10}
 & \multirow[c]{4}{*}{GPTQ} & AbsMax & 51.18 & 84.31 & 74.68 & 65.52 & 64.56 & 68.05 & 94.42 \\
 &  & 4/6 & 51.67 & 85.06 & \textbf{80.22} & 65.68 & 62.27 & 68.98 & 95.72 \\
 &  & \cellcolor{blue!15}ScaleSweep$_\mathrm{MSE}$ & \cellcolor{blue!15}52.35 & \cellcolor{blue!15}84.99 & \cellcolor{blue!15}79.48 & \cellcolor{blue!15}65.65 & \cellcolor{blue!15}63.30 & \cellcolor{blue!15}69.16 & \cellcolor{blue!15}95.96 \\
 &  & \cellcolor{blue!15}ScaleSweep & \cellcolor{blue!15}\textbf{53.25} & \cellcolor{blue!15}\textbf{86.35} & \cellcolor{blue!15}78.19 & \cellcolor{blue!15}\textbf{66.73} & \cellcolor{blue!15}\textbf{65.04} & \cellcolor{blue!15}\textbf{69.91} & \cellcolor{blue!15}\textbf{97.01} \\
\cmidrule(l){1-10}
\multirow[c]{8}{*}{WAKVQ} & \multirow[c]{4}{*}{RTN} & AbsMax & 47.44 & 82.41 & 76.34 & 64.38 & 61.17 & 66.35 & 92.06 \\
 &  & 4/6 & 46.88 & 83.02 & 77.45 & 65.21 & \textbf{63.06} & 67.12 & 93.14 \\
 &  & \cellcolor{blue!15}ScaleSweep$_\mathrm{MSE}$ & \cellcolor{blue!15}49.29 & \cellcolor{blue!15}\textbf{84.08} & \cellcolor{blue!15}77.63 & \cellcolor{blue!15}65.07 & \cellcolor{blue!15}62.04 & \cellcolor{blue!15}67.62 & \cellcolor{blue!15}93.83 \\
 &  & \cellcolor{blue!15}ScaleSweep & \cellcolor{blue!15}\textbf{51.70} & \cellcolor{blue!15}84.00 & \cellcolor{blue!15}\textbf{78.37} & \cellcolor{blue!15}\textbf{65.76} & \cellcolor{blue!15}60.85 & \cellcolor{blue!15}\textbf{68.14} & \cellcolor{blue!15}\textbf{94.54} \\
\cmidrule(l){2-10}
 & \multirow[c]{4}{*}{GPTQ} & AbsMax & 46.96 & 81.12 & 77.82 & 65.24 & 61.48 & 66.52 & 92.30 \\
 &  & 4/6 & 49.17 & 81.88 & 76.71 & 65.23 & 62.51 & 67.10 & 93.10 \\
 &  & \cellcolor{blue!15}ScaleSweep$_\mathrm{MSE}$ & \cellcolor{blue!15}50.87 & \cellcolor{blue!15}\textbf{84.76} & \cellcolor{blue!15}\textbf{78.93} & \cellcolor{blue!15}65.46 & \cellcolor{blue!15}64.88 & \cellcolor{blue!15}\textbf{68.98} & \cellcolor{blue!15}\textbf{95.71} \\
 &  & \cellcolor{blue!15}ScaleSweep & \cellcolor{blue!15}\textbf{51.53} & \cellcolor{blue!15}83.70 & \cellcolor{blue!15}76.34 & \cellcolor{blue!15}\textbf{65.91} & \cellcolor{blue!15}\textbf{64.96} & \cellcolor{blue!15}68.49 & \cellcolor{blue!15}95.03 \\
\bottomrule
\end{tabular}
\caption{Comparison results of ScaleSweep and baseline initialization methods under RTN and GPTQ across different settings on Qwen3-8B and Qwen3-4B in the non-thinking mode.} \label{tab:main_result_qwen_3_4B_8B}
\end{table*}

\end{document}